\documentclass{article}

\usepackage{microtype}
\usepackage{graphicx}
\usepackage{bbding}
\usepackage{subfigure}
\usepackage{booktabs} 

\usepackage{hyperref}
\usepackage{bbm}

\usepackage{amsmath,amsfonts,bm}

\def\eqref#1{equation~\ref{#1}}

\def\1{\bm{1}}

\def\va{{\bm{a}}}
\def\vb{{\bm{b}}}

\def\vf{{\bm{f}}}

\def\vu{{\bm{u}}}
\def\vv{{\bm{v}}}
\def\vw{{\bm{w}}}

\def\vz{{\bm{z}}}

\DeclareMathAlphabet{\mathsfit}{\encodingdefault}{\sfdefault}{m}{sl}
\SetMathAlphabet{\mathsfit}{bold}{\encodingdefault}{\sfdefault}{bx}{n}

\usepackage[accepted]{icml2023}

\usepackage{amsmath}
\usepackage{amssymb}
\usepackage{mathtools}
\usepackage{amsthm}

\usepackage[capitalize,noabbrev]{cleveref}

\theoremstyle{plain}
\newtheorem{theorem}{Theorem}[section]
\newtheorem{proposition}[theorem]{Proposition}
\newtheorem{lemma}[theorem]{Lemma}
\newtheorem{corollary}[theorem]{Corollary}
\theoremstyle{definition}

\theoremstyle{remark}

\usepackage[textsize=tiny]{todonotes}

\icmltitlerunning{From Relational Pooling to Subgraph GNNs}

\begin{document}

\twocolumn[
\icmltitle{From Relational Pooling to Subgraph GNNs:\\ A Universal Framework for More Expressive Graph Neural Networks}



\icmlsetsymbol{equal}{*}

\begin{icmlauthorlist}
\icmlauthor{Cai Zhou}{equal,1}
\icmlauthor{Xiyuan Wang}{equal,2}
\icmlauthor{Muhan Zhang}{2}
\end{icmlauthorlist}

\icmlaffiliation{1}{Department of Automation, Tsinghua University}
\icmlaffiliation{2}{Institute for Artificial Intelligence, Peking University}

\icmlcorrespondingauthor{Muhan Zhang}{muhan@pku.edu.cn}

\icmlkeywords{Graph Neural Network, Weisfeiler-Lehman test}

\vskip 0.3in
]

\newcommand{\muhan}[1]{\textcolor{red}{MZ: #1}}
\newcommand{\cai}[1]{\textcolor{blue}{CZ: #1}}
\newcommand{\wxy}[1]{{\color{red}#1}}
\newcommand{\mset}[1]{\{\!\{#1\}\!\}}
\newcommand{\bmset}[1]{\big\{\!\big\{#1\big\}\!\big\}}
\newcommand{\Bmset}[1]{\Big\{\!\Big\{#1\Big\}\!\Big\}}
\newcommand{\tuple}[1]{\left(#1 \right)}


\printAffiliationsAndNotice{\icmlEqualContribution} 

\begin{abstract}
Relational pooling (RP) is a framework for building more expressive and permutation-invariant graph neural networks (GNN). However, there is limited understanding of the exact enhancement in the expressivity of RP and its connection with the Weisfeiler–Lehman (WL) hierarchy. Starting from RP, we propose to explicitly assign labels to nodes as additional features to improve graph isomorphism distinguishing power of message passing neural networks. The method is then extended to higher-dimensional WL, leading to a novel $k,l$-WL algorithm, a more general framework than $k$-WL. We further introduce the subgraph concept into our hierarchy and propose a localized $k,l$-WL framework, incorporating a wide range of existing work, including many subgraph GNNs. Theoretically, we analyze the expressivity of $k,l$-WL w.r.t. $k$ and $l$ and compare it with the traditional $k$-WL. Complexity reduction methods are also systematically discussed to build powerful and practical $k,l$-GNN instances. We theoretically and experimentally prove that our method is universally compatible and capable of improving the expressivity of any base GNN model. Our $k,l$-GNNs achieve superior performance on many synthetic and real-world datasets, which verifies the effectiveness of our framework.
\end{abstract}

\section{Introduction}

Graph-structured data has recently revealed a significant importance in many fields, including bio-informatics, combinatorial optimization and social-network analysis, among which graph neural networks (GNNs) achieve great successes~\citep{BronsteinBLSV16, DBLP:journals/corr/abs-2003-03123, DBLPCombine}. Message passing neural network (MPNN) is one of the simplest and most commonly used GNNs~\citep{Zhou2018GraphNN}, whereas its expressivity in distinguishing non-isomorphic graphs is bounded by the one-dimensional Weisfeiler-Lehman test (1-WL)~\citep{HowPowerfulGNN, HigherorderGNN}. Therefore, designing GNNs with stronger expressivity has aroused increasing attention. 

\begin{figure}[t]
\begin{center}
\centerline{\includegraphics[page=3, width=\columnwidth]{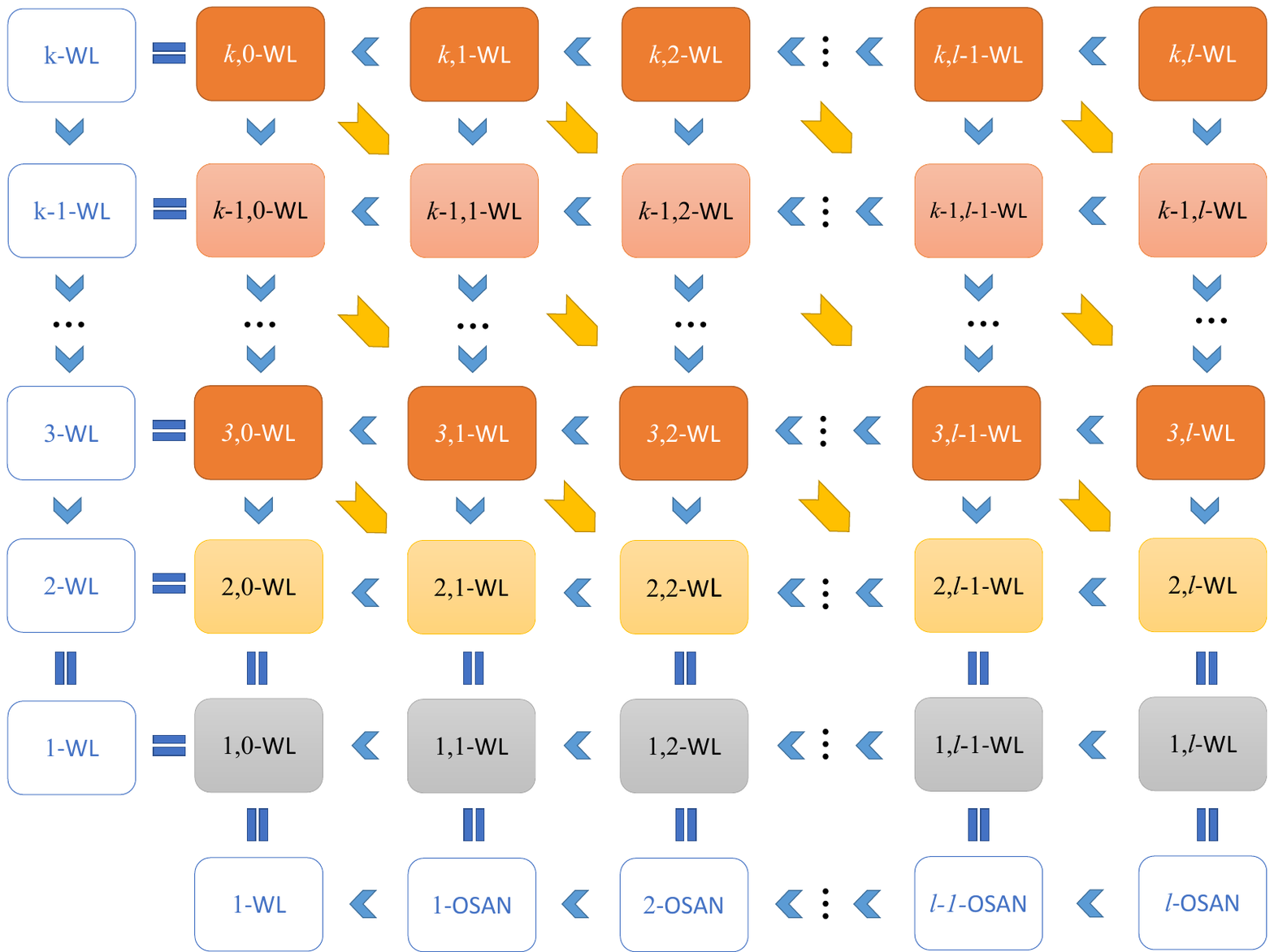}}
\caption{The expressivity hierarchy of $k,l$-WL. The blue arrows indicate \cref{k+1l>kl} and \cref{kl+1>kl}, showing that increasing $k$ and $l$ will strictly increase expressivity. The yellow arrows imply \cref{k-1l+1<=kl}, which states that $k+1,l$-WL is strictly more powerful than $k,l+1$-WL when $k\geq 2$.}
\label{expressivityfigure}
\end{center}
\vskip -0.2in
\end{figure}

Numerous approaches have been proposed to enhance GNN's expressivity. Relational Pooling (RP)~\citep{Murphy2019RelationalPF,CountSubstructures} is a framework to build powerful permutation-invariant models by symmetrizing expressive permutation-sensitive base models. Concretely, RP first feed adjacency matrix to a powerful permutation-sensitive model, like Multi-Layer Perceptron (MLP), to achieve strong expressivity. Then the permutation invariance is guaranteed by averaging or summing over representations under all permutations of node IDs (hence all permutations of adjacency matrix). However, RP is impractical for most real-world graphs due to the $O(n!)$ complexity, where $n$ is the number of nodes. Based on RP, \citet{CountSubstructures} further introduce a local version called Local Relational Pooling (LRP), which performs permutation and averaging within an induced subgraph. LRP's time complexity is reduced to the number of subgraphs $O(n^l)$, where $l$ is the subgraph size. But so far, there is a lack of theoretical analysis on the expressivity of LRP. In this paper, we propose ID-MPNN, a variation of LRP that avoids the time complexity of RP. Instead of using MLP as the base encoder, ID-MPNN runs MPNN on the whole graph. Meanwhile, to improve expressivity, $l$ nodes in the whole graph are labeled with $1,2,...,l$. Through the lens of ID-MPNN, we establish a connection between (local) Relational Pooling and subgraph GNNs.

Furthermore, by replacing MPNN with more powerful base encoders $k$-WL, we propose $k,l$-WL, a universal framework for many expressive GNNs (shown in Figure~\ref{algorithmfigure}). Intuitively, $k,l$-WL can be viewed as running $k$-WL on a graph with $l$ nodes labeled and symmetrizing over $l$. We theoretically analyze the expressivity of $k,l$-WL and build a complete expressivity hierarchy of the algorithms with different $k, l$ as shown in Figure~\ref{expressivityfigure}. As a universal framework, $k,l$-WL incorporates a wide range of existing algorithms and GNN models, including relational pooling, the original $k$-WL, many subgraph GNNs, and some other GNN extensions. 

In summary, the organization of this paper and our main contributions are as follows. 
\begin{enumerate}
    \item \cref{Sectionframework} proposes ID-MPNN to improve LRP. ID-MPNN is further extended to a general framework $k,l$-WL, which incorporates a majority of existing WL and GNN variations, including RP and many subgraph GNNs. 
    \item \cref{Theory} theoretically analyzes the algorithm's expressivity and builds a strict $k,l$-WL expressivity hierarchy, which is more general than the $k$-WL hierarchy. 
    \item \cref{sectionimprovements} discusses practical issues in our $k,l$-WL framework and proposes techniques to improve scalability. 
    \item \cref{experiments} evaluates $k,l$-WL with extensive experiments on both synthetic and real-world datasets. Our models achieve state-of-the-art results on several tasks and significantly outperforms previous works based on RP.
\end{enumerate}

\section{Related Work}\label{RelatedWork}

\paragraph{Graph Neural Network and Weisfeiler-Lehman test} Weisfeiler-Lehman tests are a classical family of algorithms to distinguish non-isomorphic graphs. Previous works have built connections between the expressivity of GNNs and WL hierarchy~\citep{HowPowerfulGNN, UnderstandingSymmetry,WLgoNeural,WLgoSparse}. We propose $k,l$-WL hierarchy, which is finer than $k$-WL and covers a wide range of existing models.

\paragraph{Subgraph GNNs} Subgraph GNNs encode a set of subgraph instead of the original graph for graph representation learning. Through careful designs, they can have both strong expressivity and good scalability. Many subgraph GNNs sample a subgraph for each node~\citep{IDAwareGNN,NGNN,GSN,SUGARSN,GNNAK}. \citet{UnderstandingSymmetry,AComplteHierarchy} upper bound the expressivity of these subgraph GNNs by 3-WL. $I^2$-GNN extracts a subgraph for each connected node pair and boosts the cycle counting power~\citep{I2GNN}. \citet{OSAN} further extract a subgraph for each $l$-tuple of nodes and propose $l$-OSAN, which is equivalent to our ID-MPNN and $2,l$-WL. We propose a more general framework $k,l$-WL to incorporate most existing subgraph GNNs.

\begin{figure}[t]
\begin{center}
\centerline{\includegraphics[width=\columnwidth]{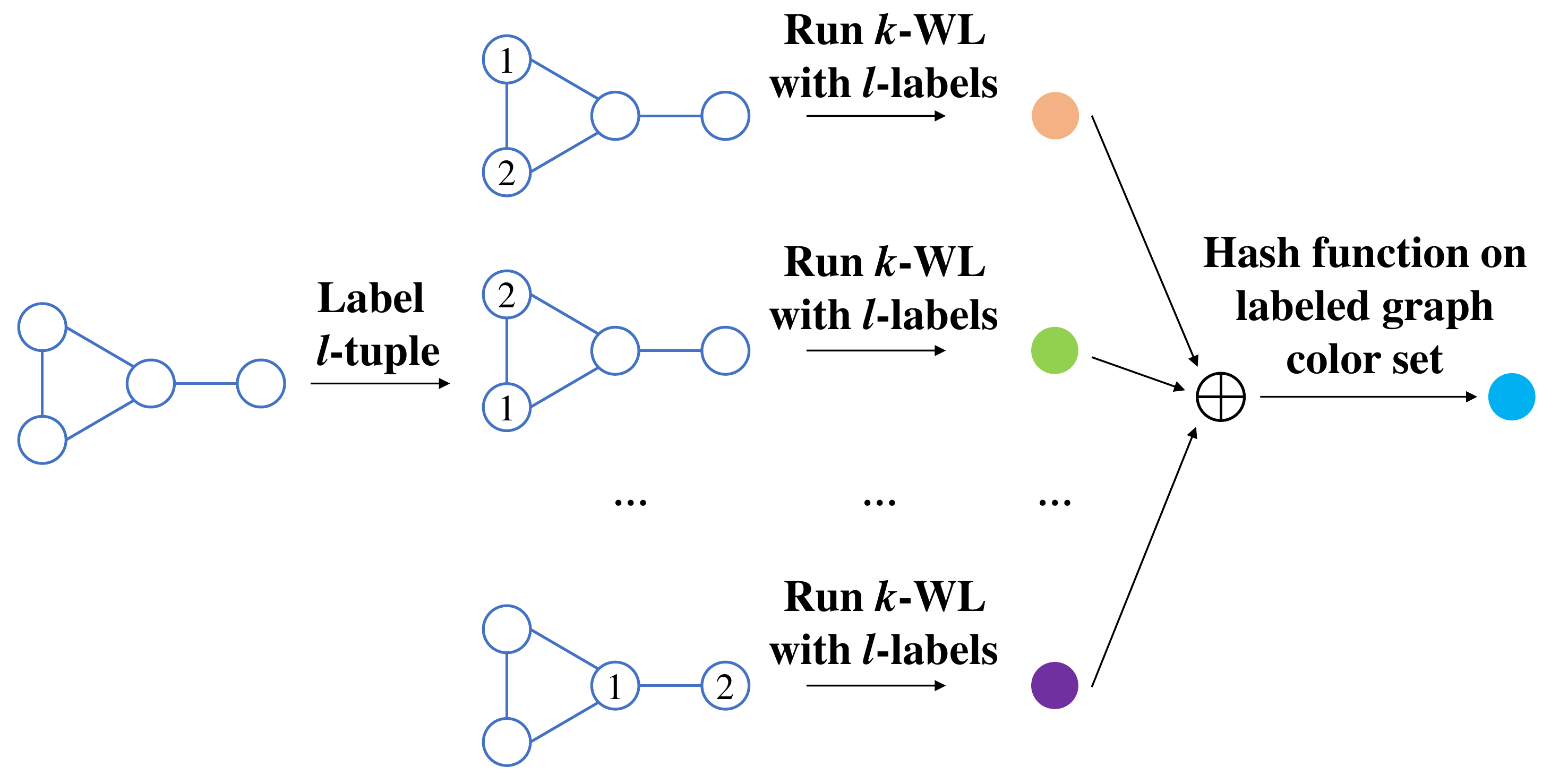}}
\caption{Procedure of $k,l$-WL algorithm. Given an input graph, we label all $l$-tuples with explicit $l$ IDs. Then $k$-WL is performed on every labeled graph, whose initialization considers the isomorphism type of $k$-tuples under the $l$ labels. Finally, a hash function maps the multiset of colors computed in all labeled graphs to output the final color of the original graph. In this example, $l=2$.}
\label{algorithmfigure}
\end{center}
\vskip -0.2in
\end{figure}

\section{Preliminary}
Given an undirected graph $G=(V, E, X)$, where $V, E$ are the node set and edge set respectively, and $X_i$ is the node feature of node $i$, let $N(v, G)=\{u\in V|(u,v)\in E\}$ denote the set of neighbors of node $v$ in graph $G$. Let $[n]$ denote the set $\{1,2,...,n\}$. Given $k$-tuple $\va\in V^k$ and $l$-tuple $\vb \in V^l$, let $\va|\!|\vb$ denote a $k+l$-tuple, the concatenation of $\va$ and $\vb$. Let $\va_i$ denote the $i$-th element in tuple $\va$, $\psi_i(\va, u)$ denote a tuple produced by replacing $\va_i$ with $u$. Let $\va_{a:b}$ denote the slice of tuple $\va$ containing $a,a+1,...,b-1$-th elements, where $a$ is omitted if $a=1$, $b$ is omitted if $b=|\va|+1$, and $|\va|$ is the length of $\va$.

Weisfeiler-Lehman test (1-WL) is a common graph isomorphism test, which also bounds the expressivity of message passing neural networks (MPNNs)~\citep{HowPowerfulGNN}. It assigns a color $c_1^{0}(v, G)$ to each node $v$ in graph $G$ initially according to $X_v$. If the graph has no node feature, the colors of all nodes are the same. Then, 1-WL iteratively updates the node colors. The $t$-th iteration is as follows.
\begin{small}
\begin{equation}
c_1^{t}\!(v, G)\!=\!\text{Hash}(c_1^{t-1}\!(v, G), \mset{c_1^{t-1}\!(u, G)|u\!\in\!N(v, G)}),
\end{equation}
\end{small}
where $c_1^{t}(v, G)$ is the color of node $v$ at the $t$-th iteration. The color of $v$ is updated by its original color and the colors of its neighbors. The color of the whole graph is the multiset of the node colors
\begin{equation}
    c_1^{t}(G)=\text{Hash}(\mset{c_1^{t}(v, G)|v\in V(G)}).
\end{equation}

There still exist non-isomorphic graphs that 1-WL cannot differentiate. $k$-dimensional Weisfeiler-Lehman test has stronger expressivity. It assigns colors to all $k$-tuples and iteratively updates them. The initial color $c_k^{0}(\vv, G)$ of tuple $\vv\in V(G)^k$ is determined by the isomorphism type of tuple $\vv$~\citep{PPGN} (see Appendix~\ref{app:proof}).
At the $t$-th iteration, the color updating scheme is
\begin{multline}
    c_k^{t}(\vv, G)=\text{Hash}(c_k^{t-1}(\vv, G), (\mset{\\
    c_k^{t-1}(\psi_i(\vv, u), G)|u\in V(G)}|i\in [k])), 
\end{multline}
where $\psi_i(\vv, u)$ means replacing the $i$-th element in $\vv$ with $u$. The color of $\vv$ is updated by its original color and the color of its high-order neighbors $\psi_i(\vv, u)$. The color of the whole graph is the multiset of all tuple colors,
\begin{equation}
    c^{t}_k(G)=\text{Hash}(\mset{c^{t}_k(\vv, G)|\vv\in V(G)^k}).
\end{equation}

Note that $k$-WL ($k\ge 2$) takes a different form from $1$-WL. Our discussion mainly focuses on $k\ge 2$ cases. Since $1$-WL has the same expressivity as $2$-WL, we can directly apply the conclusion of $2$-WL to $1$-WL.

\section{$k,l$-WL: A Universal Framework} \label{Sectionframework}

\subsection{Message passing with labels: enhancement by asymmetry}\label{sectionMPNN}

The expressivity of models built by Relational Pooling (RP) depends on the power of the base encoder before symmetrization. Some previous works use MLP \citep{CountSubstructures} and RNN \citep{PG-GNN} to capture relations between nodes. They have high expressivity but little inductive bias for graph data. Moreover, in practical settings, where Local Relational Pooling (LRP) is used, they can only encode induced subgraphs and lose the global information of graph. To solve these problems, we introduce asymmetry to MPNN by assigning nodes unique labels (which are additional features, different from the node indices only to name different nodes) and use \textit{MPNN with labels} as the base encoder.

Given an input graph $G$, MPNN with labels first assigns label $i$ (node ID) to each node $i$ as an additional feature and then runs standard message passing on the labeled graph. 
MPNN with full labels is expressive enough to encode the full graph information: MPNN can encode the multiset of neighbors into node representations. With node ID labels, each node's representation can identify the neighboring nodes connected to it. Therefore, the representation of the whole graph can identify the connectivity between nodes in the whole graph and thus enable distinguishing non-isomorphic graphs. Moreover, the standard message passing introduces inductive bias on graph data. MPNN with labels can also be easily adapted to the LRP setting. Instead of assigning all nodes unique labels, MPNN with labels can assign $1,2,..,l$ to only $l$ nodes and run message passing on the whole graph. Therefore, MPNN with labels can still capture global graph feature, unlike standard LRP only taking induced subgraphs as input. 

Our ID-MPNN combines MPNN with labels with LRP. An ID-MPNN parameterized by $l$ (called $l$-IDMPNN) explicitly assigns $l$ \textbf{unique} labels (IDs) to $l$ nodes (can be duplicated, thus $n^l$ labeled graphs in total) as an additional feature. Then, a standard message passing is performed on each labeled graph. Finally, the representations of these labeled graphs are aggregated to produce the original graph representation. A contemporary work by \citet{OSAN} also proposes similar models. However, they neither connect ID-MPNN with LRP nor extend ID-MPNN to a more general framework $k,l$-WL as in the following section. 

\subsection{$k,l$-WL: enhancement by higher dimension}\label{klWL}

So far, it is natural to ask: what if we replace the MPNN with other more powerful GNNs? Equivalently, can 1-WL be replaced by higher-dimensional WL tests on the labeled graphs? When all nodes are assigned with unique labels, even MPNN can distinguish all non-isomorphic graphs and thus using more powerful models is meaningless. However, when the number of nodes labeled is fixed, we give a positive answer: if we run $k$-WL ($k\ge 3$) with $l$ labels, it will be more powerful than $1$-WL (with $l$ labels). We name running $k$-WL on labeled graphs with $l$ IDs as $k,l$-WL, which is formally defined as follows.

\begin{enumerate}
    \item Given an $l$-tuple of nodes $\vv$ in graph $G=(V,E, X)$, the labeled graph is $G^{\vv}=(V,E,X^{\vv})$, where $\forall u\in V, X^{\vv}_u=\text{Hash}\tuple{X_u,\mset{i|\vv_i=u,i\in [l]}}$. In other words, node $\vv_i$ will have an extra label $i$. 
    \item $k,l$-WL then runs $k$-WL on each labeled graph $G^{\vv}$.
    \begin{itemize}
        \item $c_k^{0}(\vu, G^{\vv})$, the color of $k$-tuple $\vu$ in graph $G^{\vv}$ is initialized by the isomorphism type of $\vu$ in $G^{\vv}$.
        \item The tuple color at the $t$-th iteration:
        \begin{multline}
            c_k^{t}(\vu, G^{\vv})=\text{Hash}(c_k^{t-1}(\vu, G^{\vv}), \\
                (\mset{
                c_k^{t-1}(\psi_i(\vu,w), G^{\vv})|w\in V}|i\in [k])).
        \end{multline}
        \item The full graph color  at the $t$-th iteration:
        \begin{multline}
            c_k^{t}(G^{\vv})=\text{Hash}(\mset{
                c_k^{t}(\vu, G^{\vv})|\vu\in V^k}).
        \end{multline}
    \end{itemize}
    \item The color of the whole graph is produced by aggregating the colors of all labeled graphs.
    \begin{equation}
    c^{t}_{k,l}(G)\!=\!\text{Hash}\big(\mset{c_k^{t}(G^{\vv})|\vv\in V^l}\big).
    \end{equation}
\end{enumerate}

Here, we briefly explain the $k,l$-WL algorithm, a more general and powerful form of $k$-WL. The key difference between $k,l$-WL and traditional $k$-WL lies in the initialization process. In $k$-WL, by applying a hash function, two $k$-tuples will get the same initial color if and only if they are from the same isomorphism class. However, this initialization results in limited initial colors due to the limited size of isomorphism classes within $k$-tuples. For example, $1$-WL assigns all nodes the same color since there is only one isomorphism type of $1$-tuple, which further restricts the expressivity of the following steps in the algorithm. In comparison, at the initialization of $k,l$-WL, unique labels are assigned to $l$-nodes. We then assign colors to these $k$-tuples according to their isomorphism types concerning the labeled tuples. See \cref{app:proof} for details and the mathematical forms. We will show this initialization makes $k,l$-WL more expressive than $k$-WL. Then the update scheme in $k,l$-WL is the same as that of $k$-WL. 

The method to explicitly assign IDs to certain nodes aligns well with the original $k$-WL hierarchy, and we will show in \cref{Theory} that $k,l$-WL is \textbf{strictly} more powerful than $k$-WL when $l>0$. We refer our readers to \cref{app:proof} for detailed proofs, \cref{example} for an example which helps to understand better the effect of explicitly introducing IDs, and \cref{discussionID} for more insights.

Additionally, note that any $k,l$-WL algorithm has a corresponding GNN implementation, which we name as $k,l$-GNN. Theoretically, if the instance contains a base GNN encoder equivalent to $k$-WL, explicitly embeds IDs to $l$-tuples, and uses an injective pooling function, $k,l$-GNN is as powerful as $k,l$-WL. To align with existing methods, we design two practical network architectures to implement $k,l$-GNN as shown in \cref{architecture}.

\subsection{Unifying existing hierarchies}\label{sectionunify}

Here we briefly discuss the connection between our framework and previous work, including subgraph GNNs, relational pooling, GNN extensions mentioned in \citep{TheoreticalComparison}, and other methods. We refer our readers to \cref{relationshipwithother} for more details. 

Firstly, $k,l$-WL can incorporate all relational pooling (RP) and Local Relational Pooling methods, since node marking is the most general and expressive form and can simulate all other extensions \cite{TheoreticalComparison}. 

Secondly, $k,l$-WL incorporates a wide range of subgraph GNNs. \citet{AComplteHierarchy} shows that all node-based subgraph GNNs fall in one of 6 equivalent classes of Subgraph Weisfeiler-Lehman Tests. Remarkably, SWL is exactly $1,1$-WL (and equivalently, $2,1$-WL) in our framework, which reveals the connection between our work and many other subgraph GNNs unified by \citet{AComplteHierarchy}. Moreover, $k,l$-WL also incorporates some subgraph GNNs out of the scope of SWL \cite{AComplteHierarchy}, such as $I^2$-GNN \cite{I2GNN}. Our $1,2$-WL is a slightly more powerful version than $I^2$-GNN since we consider all $2$-tuples to label, while $I^2$-GNN only considers those connected $2$-tuples. $1,2$-WL can distinguish some non-isomorphic graph pairs that SWL and 3-WL fail to discriminate, and the algorithm becomes even more powerful as we increase $k$ or $l$. 

Thirdly, while a number of works such as OSAN \citep{OSAN} are a strict subclass of our framework, there are still some works cannot be incorporated directly. For example, $k,l$-WL operates independently on different labeled graphs and does not include interaction between labeled (sub)graphs as in \citet{GNNAK} and those more expressive SWL variants in \citet{AComplteHierarchy}. Introducing inter-labeled-graph message passing will increase the expressivity of $k,l$-WL, but at the price of additional computation cost. It will also be complicated to analyze its theoretical expressivity if we introduce labeled graph interactions, which we leave for future work.

Finally, our framework incorporates all four kinds of GNN extensions in \citet{TheoreticalComparison}: higher-order WL, counting substructures, injecting local information,
and marking nodes. Due to the limited space, the detailed discussion is in \cref{relationshipwithother}.

\section{The Expressivity Hierarchy of $k,l$-WL}
\label{Theory}

In this section, we theoretically analyze the expressivity of $k,l$-WL. We first define the comparison between algorithms' expressivity:

For any algorithm $A$ and $B$, we denote the final color of graph $G$ computed by them as $c_A(G)$ and $c_B(G)$, we say:
\begin{itemize}
    \item $A$ is \textbf{more powerful} than $B$ ($B\preceq A$) if for any pair of graphs $G$ and $H$, $c_A(G)=c_A(H)\Rightarrow c_B(G)=c_B(H)$. Otherwise, there exists a pair of graphs that $B$ can differentiate while $A$ cannot, denoted as $B\not\preceq A$. 
    \item $A$ is \textbf{as powerful as} $B$ ($A\cong B$) if $B\preceq A \land A\preceq B$.
    \item $A$ is \textbf{strictly more powerful} than $B$ ($B\prec A$) if $B\preceq A \ \land \ A\ncong B$, i.e., for any pair of graphs $G$ and $H$, $c_A(G)=c_A(H)\Rightarrow c_B(G)=c_B(H)$, and there exists at least one pair of graphs $H, G$ s.t. $c_B(G)=c_B(H),c_A(G)\neq c_A(H)$.
    \item $A$ and $B$ are \textbf{incomparable} ($A\nsim B$) if $A \npreceq B \ \land \ B\npreceq A$. In this case, $A$ can distinguish a pair of non-isomorphic graphs that cannot be distinguished by $B$ and vice versa.
\end{itemize}

\subsection{Connection with existing hierarchies}

A special case of $k,l$-WL is that $l=0$ and no extra labels is attached to the graph. We have 
\begin{theorem}
    $\forall k\ge 2$, $k,0$-WL $\cong$ $k$-WL.
\end{theorem}

Another special case of $k,l$-WL is the $k=1$ case. With no label, $2$-WL is of the same expressivity to $1$-WL. We find that with $l$ labels, the equality still holds.
\begin{theorem}
    $\forall l\ge 0 $, $1,l$-WL $\cong$ $2,l$-WL
\end{theorem}
where $1,l$-WL is just $l$-OSAN~\citep{OSAN}.

A variant of $k$-WL test is $k$-Folklore Weisfeiler-Lehman ($k$-FWL) test. It is known that $\forall k\ge 1$, $k$-FWL $\cong$ $k+1$-WL. With labels, the equality still holds:
\begin{theorem}
    $\forall k\ge 1, l\ge 0$, $k,l$-FWL $\cong$ $k+1,l$-WL
\end{theorem}
where $k,l$-FWL runs $k$-FWL on $l$-labeled graphs. See \cref{app:kwl} for more details.

$1,1$-WL is equivalent to the vanilla subgraph Weisfeiler-Lehman test (SWL(VS)) proposed by \citet{AComplteHierarchy}.
\begin{theorem}
    $1,1$-WL $\cong$ SWL(VS).
\end{theorem}
SWL(VS) further unifies various subgraph GNNs, like Nested GNN~\citep{NGNN} and ID-GNN~\citep{IDAwareGNN}.

\subsection{Expressivity hierarchy of $k,l$-WL}

Similar to WL tests, we can establish a hierarchy for $k,l$-WL in terms of distinguishing non-isomorphic graphs. The full hierarchy is shown in Figure~\ref{expressivityfigure}.

$\forall k\ge 2, l\ge 0$, $k, l$-WL essentially produces colors for $|V|^{k+l}$ tuples. Intuitively, increasing $k$ and $l$ will boost expressivity as more tuple colors will be computed. We show increasing $k$ and $l$ both \textbf{strictly} increases the expressivity.
\begin{theorem}\label{k+1l>kl}
    $\forall k\ge 2,l\ge 0,k,l$-WL $\prec$ $k+1,l$-WL.
\end{theorem}
\begin{theorem}\label{kl+1>kl}
    $\forall k\ge 1,l\ge 0,k,l$-WL $\prec$ $k,l+1$-WL.
\end{theorem}
However, with fixed $k+l$ and number of colors, a larger $k$ will lead to more message passing processes between tuples and stronger expressivity, shown in the following.
\begin{theorem} \label{k-1l+1<=kl}
    $\forall k\ge 2, l\ge 0$, $k, l+1$-WL $\prec$ $k+1,l$-WL.
\end{theorem}

With these theorems, we can prove a lot of useful corollaries. For example, $k, l$-WL is less expressive than $k+l$-WL:
\begin{corollary} \label{k+l>=kl}
    $\forall k\ge 2, l\ge 1$, $k,l$-WL $\prec$ $k+l$-WL. 
\end{corollary}

Moreover, $l+1$-WL is not more powerful than $2,l$-WL~\citep{OSAN}.
\begin{theorem}
    $\forall l\ge 1$, $2,l$-WL $\not\preceq$ $l+1$-WL.
\end{theorem}
When $l=2$, the above result recovers the known result of $I^2$-GNN \citep{I2GNN}. 

Besides graph isomorphism power, counting power is also an important measure of the representation capability of GNNs. We conclude that $k,l$-WL is able to count all connected substructures within $k+1$ nodes, see also \citet{EfficientlyCoutingSubstructures}. This is also verified by our experiments on substructure counting.

\section{Practical $k,l$-GNN: Improving Scalability and Compatibility}\label{sectionimprovements}

In this section, we will discuss $k,l$-GNN, the neural implementation of $k,l$-WL. $k,l$-GNN runs GNNs with the same expressivity as $k$-WL~\citep{HowPowerfulGNN,PPGN} on $l$-labeled graphs. We will discuss practical issues affecting the performance and scalability as well as our solutions. Our framework can also be applied to any other expressive GNNs to improve expressivity. 

\subsection{Implementation of $k,l$-GNN}

\begin{figure}[t]
\vskip 0.2in
\begin{center}
\centerline{\includegraphics[width=\columnwidth]{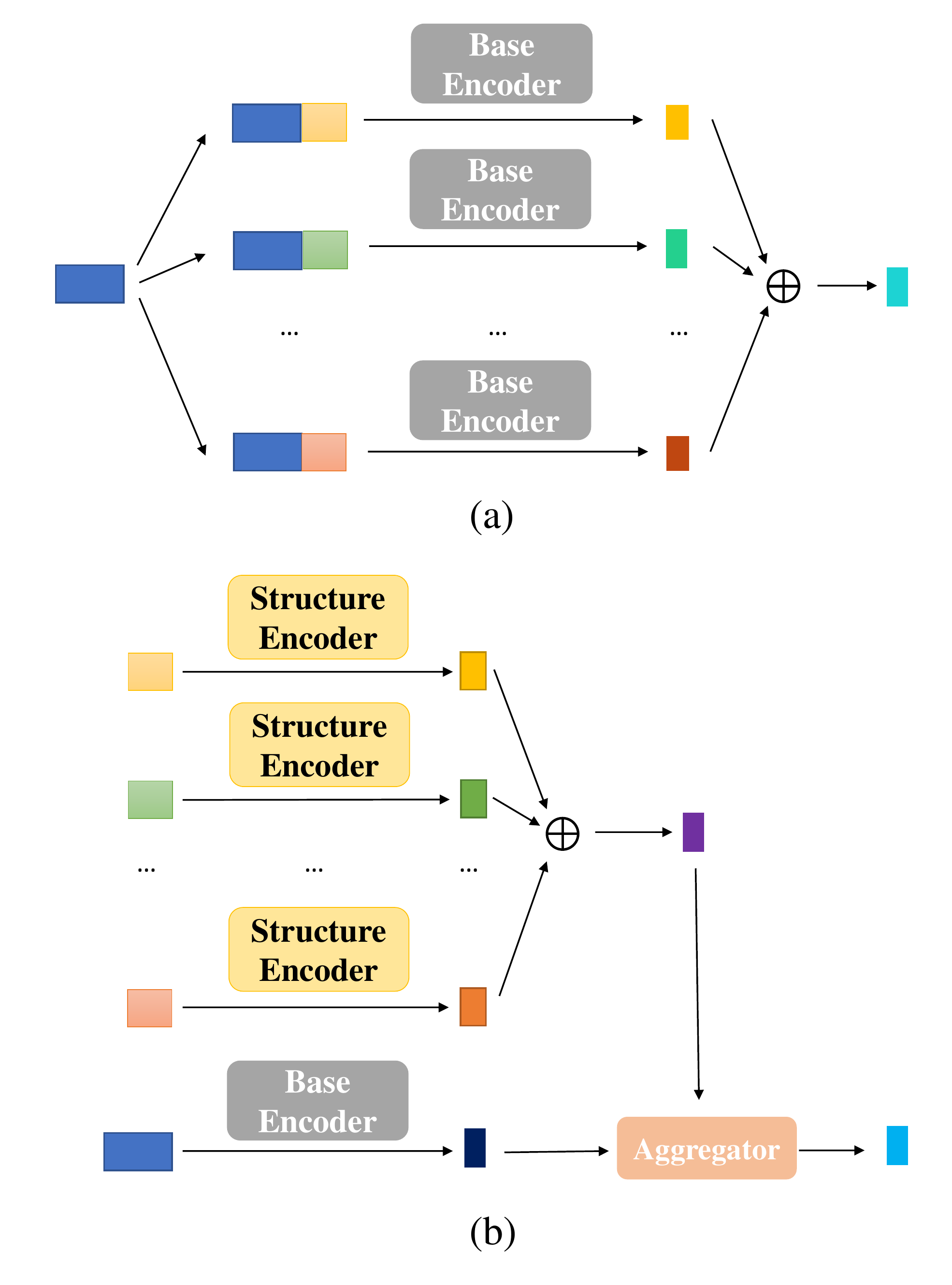}}
\caption{Two architectures of $k,l$-GNN. (a) The input graph feature (represented by blue rectangle) is duplicated for each labeled subgraph, and one base encoder jointly learns ID features (represented by squares of different colors) and graph features. (b) The graph features and ID features are parallelly learned by the base encoder and a structure encoder, respectively, which are then aggregated together and passed to downstream architectures.}
\label{architecture}
\end{center}
\vskip -0.2in
\end{figure}

Our $k,l$-WL runs $k$-WL on $l$-labeled graphs. Therefore, in implementation, we run a base encoder on all labeled graphs and pool their representations. Generally speaking, we can adopt any architecture as our base encoder and improve its expressive power via labeling. Specifically, when the base encoder has the same expressivity as $k$-WL, it reduces to our standard $k,l$-GNN. The base encoder, however, does not need to have exactly the same expressive power as certain $k$-WL in practice, since we can always upper bound the expressive power of the lifted model by some $k,l$-WL. In other words, our method can lift the expressivity of many existing architectures through labeling.

To make our framework applicable to any model, we propose two architectures to lift the expressive power of a base model. For convenience, we first suppose that the base encoder has $k$-WL-equivalent expressivity, and we will lift its expressivity to $k,l$-WL. \textit{Architecture (a)} in \cref{architecture} (a) exactly simulates the original form of $k,l$-WL: The input node feature matrix (represented by blue rectangles) is replicated to the number of labeled graphs and concatenated (or by other combination methods like add) with the corresponding label features. Then the original node features and label features are jointly learned by the same $k$-WL-equivalent GNN (base encoder). If the pooling function is injective (like Deepset~\citep{deepset}), architecture (a) fully preserves the expressivity of $k,l$-WL. \textit{Architecture (b)} in \cref{architecture} (b) is slightly different from the original $k,l$-WL, which learns original node features and new label features by two models separately. The $k$-WL-equivalent base encoder learns the original node features only once. An extra $k'$-WL ($k'\neq k$ is permitted) equivalent \textit{structure encoder} learns label features without node features in each $l$-labeled graph. The $l$-labeled graph representations are then aggregated by an aggregator and concatenated to the representation of the original graph. We call it structure encoder since label features do not introduce information more than graph structure.

Theoretical analysis and comparison between these two architectures are presented in \cref{architecturediscussion}. In short, architecture (b) is less expressive but more scalable.  As the number of labeled graphs is $n^l$, and the complexity of $k$-WL-equivalent GNN can be $n^k$, architecture (a) can preserve the expressivity of $k,l$-WL with a complexity of $O(n^{k+l})$. Architecture (b) can use different structure encoder and base encoder, i.e. $k'< k$. Then the complexity of architecture (b) can be reduced to $O(n^k+n^{k'+l})$. Moreover, we experimentally find that architecture (b) tends to outperform (a). One intuition behind is that we may need two different sets of parameters to learn label feature and node features, respectively. 



In summary, architecture (b) is designed to reduce the complexity at the cost of losing part of expressive power compared with architecture (a), the original GNN implementation of $k,l$-WL. However, the decoupled architecture tends to perform better in real-world tasks. Despite their differences, we emphasize that both architectures can boost the expressivity. 

Below, we propose several $k,l$-GNN instances parameterized by different $k$ and $l$, including ID-MPNN and ID-PPGN. We as well lift some other base models through our framework, such as ID-Transformer. Note that the following instances can all adopt either architecture (a) or (b), depending on application scenarios. 

\paragraph{ID-MPNN} An $l$-IDMPNN is an instance of $1,l$-WL, since Message Passing Neural Network (MPNN) is equivalent to $1$-WL \citep{HowPowerfulGNN}. The model is easy to implement but reveals strong expressivity as $l$ increases. One can easily verify that ID-MPNN incorporates Identity-aware GNN ($l=1$), $I^2$-GNN ($l=2$) and $l$-OSAN. With our implementation, ID-MPNN outperforms the above models experimentally. 

\paragraph{ID-PPGN} For $3$-WL equivalent base encoders, we select PPGN \citep{PPGN} as our base encoder. An ID-PPGN with $l$ labels is as powerful as $3,l$-WL, which is strictly more powerful than $3$-WL when $l>0$. 

\paragraph{ID-Transformer} We also apply our framework to graph transformers. It is not a $k,l$-GNN strictly as graph transformer is not $k$-WL equivalent. However, our framework is general enough to take any base graph learning model. 

Additionally, many techniques can be applied in the implementation of $k,l$-GNN, such as positional encodings (PE) and structure encodings (SE) \citep{GPS}, GNN as kernel techniques \citep{GNNAK}, etc. In conclusion, our method is a universal framework to improve the expressivity of base models while being compatible with many other methods and techniques. Please refer to \cref{implementationdetail} for more implementation details. In the next section, we discuss how to reduce the complexity when $l$ is large.

\subsection{Labeled graph selection}

As $k,l$-GNN runs a base encoder on each labeled graph, the total complexity is proportional to the number of labeled graphs. This section focuses on how to select a subset of labeled graphs. We also discuss how to segregate a subgraph from the labeled graph and thus reduce the size of labeled graphs in \cref{sectionlocality}, as subgraph size also affects the scalability and reduced subgraph size is important for the successes of many subgraph GNNs~\citep{NGNN}.

We list some labeled graph selection strategies as follows. Among them, random sampling and node-based policies are commonly used. We propose two new sampling policies: constraint-based policy and hierarchical policy. 

\begin{itemize}
    \item Random sampling. This is one of the most common methods in existing subgraph GNNs. Statistically, the graph representation is permutation invariant and unbiased but with a variance among different samplings. We can conduct parallel samplings for variance reduction and use mean/median/voting as the final representation. The complexity of randomly sampling $l$-tuples is $O(\alpha n^l)$, where $\alpha$ is the sampling rate.
    \item Node-based policies. This is another family of invariant sampling methods. For example, we can extract a K-hop ego-net for each root node and select all $l$-tuples in the ego-net with the root node always being the first node in each $l$-tuple. The complexity is $O(nm^{l-1})$, where $m$ is the average size of ego-nets. 
    \item Constraint-based policies. These methods search within all possible $l$-tuples and filter out those failing to meet certain constraints. For example, if we upper-bound the shortest path distance $3$ between any pair of nodes in a $6$-tuple, we can sample many $6$-rings while excluding any $6$-paths. Compared with node-based policy, constraint-based methods do not sample the same induced subgraphs repetitively and enjoy a higher design freedom. In implementation, this can be efficiently implemented by dynamic programming, e.g., Floyd-Warshall algorithms. The number of subgraphs depends on the graph and constraints.
    \item Hierarchical policies. Theses methods hierarchically select subgraphs. We can use algorithms like min-cut or node clustering (e.g., spectral clustering) to divide the graph into clusters with an average size $m$. Labeled nodes are then selected only within each cluster, resulting in an average complexity of $O(m^l\cdot \frac{n}{m})$. Since $m\ll n$, the hierarchical policy can significantly reduce the number of subgraphs while still being able to encode sufficient local structure information.
\end{itemize}


In most real-world task experiments, we use constraint-based and hierarchical policies, achieving impressive experimental results at a low computation complexity. See \cref{ablation} for an ablation study on different sampling policies.





\paragraph{Learn to sample and Learn to label}

Aside from traversing or sampling according to certain rules, \citet{OSAN} use Implicit-MLE framework \citep{IMLE} (which allows us to back-propagate through continuous-discrete architecture) to sample subgraphs in a data-driven fashion. With this method, $k,l$-GNN can also learn to label tuples and sample subgraphs to minimize the target loss function in a data-driven manner. See \cref{ablation} for more details and experimental results. 


\begin{table*}[t]
\caption{QM9 results (MAE $\downarrow$)}
\label{QM9}
\vskip 0.15in
\begin{center}
\begin{small}
\begin{sc}
\begin{tabular}{l|ccccc|c}
\toprule
Target & MPNN & DTNN & DeepLRP & PPGN & Nested GNN & IDMPNN(l=4)\\
\midrule
$\mu$ & 0.358 & \textbf{0.244} & 0.364 & \textbf{\textcolor{red}{0.231}} & 0.433 & 0.398 \\
$\alpha$ & 0.89 & 0.95 & 0.298 & 0.382 & \textbf{0.265} & \textbf{\textcolor{red}{0. 226}} \\
$\epsilon_{\textrm{HOMO}}$ & 0.00541 & 0.00388 & \textbf{\textcolor{red}{0.00254}} & 0.00276 & 0.00279 & \textbf{0.00263}\\
$\epsilon_{\textrm{LUMO}}$ & 0.00623 & 0.00512 & \textbf{0.00277} & 0.00287 & \textbf{\textcolor{red}{0.00276}} & 0.00286\\
$\Delta \epsilon$  & 0.0066 & 0.0112 & \textbf{\textcolor{red}{0.00353}} & 0.00406 & \textbf{0.00390} & 0.00398\\
$\langle R^2 \rangle$ & 28.5 & 17.0 & 19.3 & \textbf{16.7} & 20.1 & \textbf{\textcolor{red}{10.4}}\\
$\textrm{ZPVE}$ & 0.00216 & 0.00172 & 0.00055 & 0.00064 & \textbf{0.00015} & \textbf{\textcolor{red}{0.00013}}\\
$U_0$ & 2.05 & 2.43 & 0.413 & 0.234 & \textbf{0.205} & \textbf{\textcolor{red}{0.0189}}\\
$U$ & 2.00 & 2.43 & 0.413 & 0.234 & \textbf{0.200} & \textbf{\textcolor{red}{0.0152}}\\
$H$ & 2.02 & 2.43 & 0.413 & \textbf{0.229} & 0.249 & \textbf{\textcolor{red}{0.0160}}\\
$G$ & 2.02 & 2.43 & 0.413 & \textbf{0.238} & 0.253 & \textbf{\textcolor{red}{0.0159}}\\
$c_{\textrm{v}}$ & 0.42 & 0.27 & 0.129 & 0.184 & \textbf{\textcolor{red}{0.0811}} & \textbf{0.0890}\\

\bottomrule
\end{tabular}
\end{sc}
\end{small}
\end{center}
\vskip -0.1in
\end{table*}


\section{Experiments}
\label{experiments}

In this section, we conduct experiments on both synthetic and real-world tasks to verify our models' expressivity and real-world performance. 

\subsection{Graph isomorphism task}
\paragraph{Dataset}
We select two synthetic datasets, EXP and SR25, to empirically verify our models' expressivity for graph isomorphism tasks. EXP \citep{EXP} contains 600 pairs of non-isomorphic graphs that 1-WL and 2-WL fail to distinguish. SR25 \citep{SR25} contains 15 non-isomorphic strongly regular graphs (i.e., 105 non-isomorphic pairs) that 3-WL fails to distinguish. An accuracy of $50\%$ on EXP and $6.67\%$ on SR25 suggest the model fails to distinguish any non-isomorphic graphs in the dataset.

\paragraph{Models}

For baseline models, we choose GIN, PNA \citep{PNA}, Identity-aware GNN \citep{IDAwareGNN}, GIN-AK+ \citep{GNNAK} and PPGN \citep{PPGN}. In comparison, we choose our ID-MPNN and ID-PPGN to understand the expressivity hierarchy better.

\paragraph{Results}

The results are shown in \cref{Simulation}. When the number of IDs $l\geq 2$, ID-MPNN and ID-PPGN achieve perfect performance on the two datasets. In comparison, all other models fail on the SR25 dataset. By comparing the results of GIN and $l$-IDMPNN as well as PPGN and $l$-IDPPGN, we verify that our framework can improve expressivity. Concretely, we have the following conclusions:
\begin{itemize}
    \item $1,1$-WL is more powerful than $1$-WL and $2$-WL.
    \item $1,2$-WL and $3,2$-WL can distinguish some non-isomorphic graph pairs that are indistinguishable by $3$-WL.
\end{itemize}
These results are consistent with our theoretical analysis in \cref{Theory}.

\begin{table}[t]
\caption{Synthetic dataset performances}
\label{Simulation}
\vskip 0.15in
\begin{center}
\begin{small}
\begin{sc}
\begin{tabular}{ccc}
\toprule
Model & EXP (Acc$\uparrow$) & SR25 (Acc$\uparrow$) \\
\midrule
GIN & 50 & 6.67 \\
PNA & 50 & 6.67 \\
ID-Aware GNN & 100 & 6.67\\
GIN-AK+ & 100 & 6.67\\
PPGN & 100 & 6.67\\
\midrule
$l$-IDMPNN($l\geq2$) & 100 & 100 \\
$l$-IDPPGN($l\geq2$) & 100 & 100\\
\bottomrule
\end{tabular}
\end{sc}
\end{small}
\end{center}
\vskip -0.1in
\end{table}

\subsection{Substructure counting}

\begin{table}[t]
\caption{Counting substructures (MAE $\downarrow$)}
\label{Counting}
\vskip 0.15in
\begin{center}
\begin{small}
\begin{sc}
\resizebox{1.\columnwidth}{!}{
\begin{tabular}{lcccc}
\toprule
Model & Tri. & Tailed Tri. & Star & Chordal Cycle\\
\midrule
GCN & 4.19E{-1} & 3.25E{-1} & 1.80E{-1} & 2.82E{-1}\\
KP-GIN+ & 3.77E{-2} & 3.14E{-2} & 2.40E{-3} & 2.58E{-2}\\
GIN-AK+ & 1.23E{-2} & 1.12E{-2} & 1.50E{-2} & 1.26E{-2}\\
DeepLRP & \textbf{1.76E{-7}} & \textbf{1.41E{-5}} & \textbf{1.41E{-5}} & \textbf{9.81E{-5}}\\
\midrule
IDMPNN & \textcolor{red}{\textbf{5.86E-46}} & \textcolor{red}{\textbf{2.25E{-7}}} & \textcolor{red}{\textbf{7.67E{-6}}} & \textcolor{red}{\textbf{1.49E{-45}}} \\
\bottomrule
\end{tabular}
}
\end{sc}
\end{small}
\end{center}
\vskip -0.1in
\end{table}

\begin{table}[t]
\caption{Zinc12K results (MAE $\downarrow$)}
\label{Zinc}
\vskip 0.15in
\begin{center}
\begin{small}
\begin{sc}
\begin{tabular}{l c}
\toprule
Method & Test MAE \\
\midrule
GIN  & $0.163\pm0.004$\\
PNA & $0.188\pm0.004$\\
GSN & $0.115\pm 0.012$\\
DeepLRP & $0.223\pm0.008$\\
OSAN & $0.187\pm0.004$\\
KP-GIN+ & $0.119\pm0.002$\\
GNN-AK+ & $0.080\pm0.001$\\
CIN & $0.079\pm0.006$\\
GPS & $\textbf{0.070}\pm0.004$\\
\midrule
4-IDMPNN & $0.083\pm0.003$\\
3-IDMPNN & $0.085\pm0.003$\\
\bottomrule
\end{tabular}
\end{sc}
\end{small}
\end{center}
\vskip -0.1in
\end{table}

\begin{table}[t]
\caption{ogbg-molhiv results (AUC $\uparrow$)}
\label{OGB}
\vskip 0.15in
\begin{center}
\begin{small}
\begin{sc}
\begin{tabular}{l c}
\toprule
Method & Test AUC \\
\midrule
PNA  & $79.05\pm1.32$\\
DeepLRP & $77.19\pm1.40$\\
NGNN & $78.34\pm 1.86$\\
KP-GIN+-VN & $78.40\pm0.87$\\
$I^2$ -GNN & $78.68\pm0.93$\\
CIN & $\textbf{80.94}\pm0.57$\\
SUN(EGO) & $80.03\pm0.55$\\
\midrule
4-IDMPNN & $79.31\pm0.63$\\
\bottomrule
\end{tabular}
\end{sc}
\end{small}
\end{center}
\vskip -0.1in
\end{table}

\paragraph{Dataset} To verify our model's expressivity of counting substructures, we evaluate on random regular graph dataset \citep{CountSubstructures}. There are four target substructures: triangle, tailed triangle, star and chordal-cycle. Test MAE measures the results.

\paragraph{Models} We choose GCN, KP-GIN+ \citep{Khop}, GIN-AK+ \citep{GNNAK}, and DeepLRP \citep{CountSubstructures} as the baseline models. 
We use $l$-IDMPNN, and additionally restricts message passing only on the labeled tuples, where $l$ is the size of the target substructure. 

\paragraph{Results} The results are shown in \cref{Counting}. Our model achieves state-of-the-art performance on all tasks, and the test MAE is nearly $0$. This verifies our theoretical results that $1,l$-GNN can completely count substructures within $l$ nodes.

\subsection{Molecular properties prediction}

\paragraph{Dataset} For real-world tasks, we choose three popular molecular property prediction datasets: QM9 \citep{QM9}, ZINC \citep{Zinc} and ogbg-molhiv \citep{OGB}. QM9 is a graph property regression dataset containing 130k small molecules and 19 regression targets, such as the energy of the molecules. We follow a commonly used training/validation/test split ratio of 0.8/0.1/0.1, and the results of the first 12 targets are reported. ZINC12k is a subset of ZINC250k containing 12k molecules. The task is also molecular property (constrained solubility) regression. ogbg-molhiv contains 41k molecules for graph binary classification (whether a molecule inhibits HIV virus replication or not). We use the official split for ZINC and ogbg-molhiv.

\paragraph{Models} For QM9 dataset, MPNN, DTNN \citep{DTNN}, DeepLRP~\citep{CountSubstructures}, PPGN~\citep{PPGN} and Nested GNN \citep{NGNN} are chosen as baseline models. For ZINC12k, we choose GIN~\citep{HowPowerfulGNN}, PNA \citep{PNA}, DeepLRP~\citep{CountSubstructures}, OSAN~\citep{OSAN}, KP-GIN+ \citep{Khop}, GNN-AK+ \citep{GNNAK}, CIN \citep{CIN} and GPS \citep{GPS} for comparison. For ogbg-mohiv, PNA~\citep{PNA}, DeepLRP~\citep{CountSubstructures}, NGNN~\citep{NGNN}, KP-GIN~\citep{Khop}, $I^2$-GNN \citep{I2GNN}, CIN~\citep{CIN} and SUN(EGO) \citep{UnderstandingSymmetry} are selected.

\paragraph{Results} Our $4$-IDMPNN achieves superior performance in 7 out of 12 tasks on QM9 dataset (\cref{QM9}), while results for the remaining targets are also highly competitive. On ZINC12k (\cref{Zinc}) and ogbg-molhiv (\cref{OGB}), although IDMPNN does not achieve the best results, it is still comparable to the state-of-the-art models. Moreover, we do not use any additional features or pretraining in any datasets, reflecting the power of our model. 
This suggests that our method can effectively enhance the performance of base encoders for real-world tasks in addition to increasing the expressivity.
While our $k,l$-GNN framework captures DeepLRP, GSN and OSAN, we find instances such as $4$-IDMPNN that greatly surpass these works in real-world tasks.

\section{Conclusions}

In this work, we establish a novel $k,l$-WL framework that explicitly assigns labels to $l$ nodes while running a $k$-WL algorithm. We theoretically analyze the expressivity hierarchy of $k,l$-WL, which incorporates many existing relational pooling methods and subgraph GNNs. Due to its strong compatibility, our framework can improve the expressivity of any base model by just augmenting ID features on (sampled) subgraphs. Various acceleration methods are also discussed to build practical, effective models. Some of our $k,l$-GNN instancees achieve state-of-the-art performance on several synthetic and real-world tasks, verifying the power of our framework.

\section*{Acknowledge}

This project is supported in part by the National Key Research and Development Program of China (No. 2021ZD0114702).

\bibliography{ref}

\begin{thebibliography}{37}
\providecommand{\natexlab}[1]{#1}
\providecommand{\url}[1]{\texttt{#1}}
\expandafter\ifx\csname urlstyle\endcsname\relax
  \providecommand{\doi}[1]{doi: #1}\else
  \providecommand{\doi}{doi: \begingroup \urlstyle{rm}\Url}\fi

\bibitem[Abboud et~al.(2020)Abboud, Ceylan, Grohe, and Lukasiewicz]{EXP}
Abboud, R., Ceylan, {\.I}.~{\.I}., Grohe, M., and Lukasiewicz, T.
\newblock The surprising power of graph neural networks with random node
  initialization.
\newblock \emph{CoRR}, abs/2010.01179, 2020.
\newblock URL \url{https://arxiv.org/abs/2010.01179}.

\bibitem[Balcilar et~al.(2021)Balcilar, H{\'e}roux, Ga{\"u}z{\`e}re, Vasseur,
  Adam, and Honeine]{SR25}
Balcilar, M., H{\'e}roux, P., Ga{\"u}z{\`e}re, B., Vasseur, P., Adam, S., and
  Honeine, P.
\newblock Breaking the limits of message passing graph neural networks.
\newblock In \emph{International Conference on Machine Learning}, 2021.

\bibitem[Bodnar et~al.(2021)Bodnar, Frasca, Otter, Wang, Lio’, Mont{\'u}far,
  and Bronstein]{CIN}
Bodnar, C., Frasca, F., Otter, N., Wang, Y.~G., Lio’, P., Mont{\'u}far, G.,
  and Bronstein, M.~M.
\newblock Weisfeiler and lehman go cellular: Cw networks.
\newblock In \emph{Neural Information Processing Systems}, 2021.

\bibitem[Bouritsas et~al.(2023)Bouritsas, Frasca, Zafeiriou, and
  Bronstein]{GSN}
Bouritsas, G., Frasca, F., Zafeiriou, S., and Bronstein, M.~M.
\newblock Improving graph neural network expressivity via subgraph isomorphism
  counting.
\newblock \emph{IEEE Transactions on Pattern Analysis and Machine
  Intelligence}, 45\penalty0 (1):\penalty0 657--668, 2023.
\newblock \doi{10.1109/TPAMI.2022.3154319}.

\bibitem[Bronstein et~al.(2016)Bronstein, Bruna, LeCun, Szlam, and
  Vandergheynst]{BronsteinBLSV16}
Bronstein, M.~M., Bruna, J., LeCun, Y., Szlam, A., and Vandergheynst, P.
\newblock Geometric deep learning: going beyond euclidean data.
\newblock \emph{CoRR}, abs/1611.08097, 2016.
\newblock URL \url{http://arxiv.org/abs/1611.08097}.

\bibitem[Cai et~al.(1992)Cai, F{\"{u}}rer, and Immerman]{PebbleGame}
Cai, J., F{\"{u}}rer, M., and Immerman, N.
\newblock An optimal lower bound on the number of variables for graph
  identification.
\newblock \emph{Comb.}, 12\penalty0 (4):\penalty0 389--410, 1992.

\bibitem[Chen et~al.(2020)Chen, Chen, Villar, and Bruna]{CountSubstructures}
Chen, Z., Chen, L., Villar, S., and Bruna, J.
\newblock Can graph neural networks count substructures?
\newblock \emph{CoRR}, abs/2002.04025, 2020.
\newblock URL \url{https://arxiv.org/abs/2002.04025}.

\bibitem[Corso et~al.(2020)Corso, Cavalleri, Beaini, Lio’, and
  Velickovic]{PNA}
Corso, G., Cavalleri, L., Beaini, D., Lio’, P., and Velickovic, P.
\newblock Principal neighbourhood aggregation for graph nets.
\newblock \emph{ArXiv}, abs/2004.05718, 2020.

\bibitem[Dai et~al.(2017)Dai, Khalil, Zhang, Dilkina, and Song]{DBLPCombine}
Dai, H., Khalil, E.~B., Zhang, Y., Dilkina, B., and Song, L.
\newblock Learning combinatorial optimization algorithms over graphs.
\newblock \emph{CoRR}, abs/1704.01665, 2017.
\newblock URL \url{http://arxiv.org/abs/1704.01665}.

\bibitem[Dwivedi et~al.(2020)Dwivedi, Joshi, Laurent, Bengio, and
  Bresson]{Zinc}
Dwivedi, V.~P., Joshi, C.~K., Laurent, T., Bengio, Y., and Bresson, X.
\newblock Benchmarking graph neural networks.
\newblock \emph{ArXiv}, abs/2003.00982, 2020.

\bibitem[Feng et~al.(2022)Feng, Chen, Li, Sarkar, and Zhang]{Khop}
Feng, J., Chen, Y., Li, F., Sarkar, A., and Zhang, M.
\newblock How powerful are k-hop message passing graph neural networks.
\newblock \emph{ArXiv}, abs/2205.13328, 2022.

\bibitem[Frasca et~al.(2022)Frasca, Bevilacqua, Bronstein, and
  Maron]{UnderstandingSymmetry}
Frasca, F., Bevilacqua, B., Bronstein, M., and Maron, H.
\newblock Understanding and extending subgraph gnns by rethinking their
  symmetries.
\newblock \emph{ArXiv}, abs/2206.11140, 2022.

\bibitem[Hamilton et~al.(2017)Hamilton, Ying, and Leskovec]{GraphSage}
Hamilton, W.~L., Ying, Z., and Leskovec, J.
\newblock Inductive representation learning on large graphs.
\newblock In \emph{NIPS}, 2017.

\bibitem[Hu et~al.(2020)Hu, Fey, Zitnik, Dong, Ren, Liu, Catasta, and
  Leskovec]{OGB}
Hu, W., Fey, M., Zitnik, M., Dong, Y., Ren, H., Liu, B., Catasta, M., and
  Leskovec, J.
\newblock Open graph benchmark: Datasets for machine learning on graphs.
\newblock \emph{ArXiv}, abs/2005.00687, 2020.

\bibitem[{Huang} et~al.(2022){Huang}, {Peng}, {Ma}, and {Zhang}]{I2GNN}
{Huang}, Y., {Peng}, X., {Ma}, J., and {Zhang}, M.
\newblock {Boosting the Cycle Counting Power of Graph Neural Networks with
  I$^2$-GNNs}.
\newblock \emph{arXiv e-prints}, art. arXiv:2210.13978, October 2022.
\newblock \doi{10.48550/arXiv.2210.13978}.

\bibitem[Huang et~al.(2022)Huang, Wang, Li, and He]{PG-GNN}
Huang, Z., Wang, Y., Li, C., and He, H.
\newblock Going deeper into permutation-sensitive graph neural networks, 2022.
\newblock URL \url{https://arxiv.org/abs/2205.14368}.

\bibitem[Klicpera et~al.(2020)Klicpera, Gro{\ss}, and
  G{\"{u}}nnemann]{DBLP:journals/corr/abs-2003-03123}
Klicpera, J., Gro{\ss}, J., and G{\"{u}}nnemann, S.
\newblock Directional message passing for molecular graphs.
\newblock \emph{CoRR}, abs/2003.03123, 2020.
\newblock URL \url{https://arxiv.org/abs/2003.03123}.

\bibitem[Maron et~al.(2019)Maron, Ben-Hamu, Serviansky, and Lipman]{PPGN}
Maron, H., Ben-Hamu, H., Serviansky, H., and Lipman, Y.
\newblock Provably powerful graph networks.
\newblock \emph{ArXiv}, abs/1905.11136, 2019.

\bibitem[Morris et~al.(2018)Morris, Ritzert, Fey, Hamilton, Lenssen, Rattan,
  and Grohe]{HigherorderGNN}
Morris, C., Ritzert, M., Fey, M., Hamilton, W.~L., Lenssen, J.~E., Rattan, G.,
  and Grohe, M.
\newblock Weisfeiler and leman go neural: Higher-order graph neural networks.
\newblock \emph{CoRR}, abs/1810.02244, 2018.
\newblock URL \url{http://arxiv.org/abs/1810.02244}.

\bibitem[Morris et~al.(2019)Morris, Ritzert, Fey, Hamilton, Lenssen, Rattan,
  and Grohe]{WLgoNeural}
Morris, C., Ritzert, M., Fey, M., Hamilton, W.~L., Lenssen, J.~E., Rattan, G.,
  and Grohe, M.
\newblock Weisfeiler and leman go neural: Higher-order graph neural networks.
\newblock In \emph{{AAAI} Conference on Artificial Intelligence}, pp.\
  4602--4609. {AAAI} Press, 2019.

\bibitem[Morris et~al.(2020)Morris, Rattan, and Mutzel]{WLgoSparse}
Morris, C., Rattan, G., and Mutzel, P.
\newblock Weisfeiler and leman go sparse: Towards scalable higher-order graph
  embeddings.
\newblock In \emph{Advances in Neural Information Processing Systems}, 2020.

\bibitem[Murphy et~al.(2019)Murphy, Srinivasan, Rao, and
  Ribeiro]{Murphy2019RelationalPF}
Murphy, R.~L., Srinivasan, B., Rao, V.~A., and Ribeiro, B.
\newblock Relational pooling for graph representations.
\newblock \emph{ArXiv}, abs/1903.02541, 2019.

\bibitem[Niepert et~al.(2021)Niepert, Minervini, and Franceschi]{IMLE}
Niepert, M., Minervini, P., and Franceschi, L.
\newblock Implicit mle: Backpropagating through discrete exponential family
  distributions.
\newblock \emph{ArXiv}, abs/2106.01798, 2021.

\bibitem[Papp \& Wattenhofer(2022)Papp and Wattenhofer]{TheoreticalComparison}
Papp, P.~A. and Wattenhofer, R.
\newblock A theoretical comparison of graph neural network extensions.
\newblock \emph{ArXiv}, abs/2201.12884, 2022.

\bibitem[Qian et~al.(2022)Qian, Rattan, Geerts, Morris, and Niepert]{OSAN}
Qian, C., Rattan, G., Geerts, F., Morris, C., and Niepert, M.
\newblock Ordered subgraph aggregation networks.
\newblock \emph{ArXiv}, abs/2206.11168, 2022.

\bibitem[Ramakrishnan et~al.(2014)Ramakrishnan, Dral, Rupp, and von
  Lilienfeld]{QM9}
Ramakrishnan, R., Dral, P.~O., Rupp, M., and von Lilienfeld, O.~A.
\newblock Quantum chemistry structures and properties of 134 kilo molecules.
\newblock \emph{Scientific Data}, 1, 2014.

\bibitem[Ramp{\'a}sek et~al.(2022)Ramp{\'a}sek, Galkin, Dwivedi, Luu, Wolf, and
  Beaini]{GPS}
Ramp{\'a}sek, L., Galkin, M., Dwivedi, V.~P., Luu, A.~T., Wolf, G., and Beaini,
  D.
\newblock Recipe for a general, powerful, scalable graph transformer.
\newblock \emph{ArXiv}, abs/2205.12454, 2022.

\bibitem[Sun et~al.(2021)Sun, Peng, Li, Wu, Ning, Yu, and He]{SUGARSN}
Sun, Q., Peng, H., Li, J., Wu, J., Ning, Y., Yu, P.~S., and He, L.
\newblock Sugar: Subgraph neural network with reinforcement pooling and
  self-supervised mutual information mechanism.
\newblock \emph{Proceedings of the Web Conference 2021}, 2021.

\bibitem[Wu et~al.(2017)Wu, Ramsundar, Feinberg, Gomes, Geniesse, Pappu,
  Leswing, and Pande]{DTNN}
Wu, Z., Ramsundar, B., Feinberg, E.~N., Gomes, J., Geniesse, C., Pappu, A.~S.,
  Leswing, K., and Pande, V.~S.
\newblock Moleculenet: a benchmark for molecular machine learning†
  †electronic supplementary information (esi) available. see doi:
  10.1039/c7sc02664a.
\newblock \emph{Chemical Science}, 9:\penalty0 513 -- 530, 2017.

\bibitem[Xu et~al.(2018)Xu, Hu, Leskovec, and Jegelka]{HowPowerfulGNN}
Xu, K., Hu, W., Leskovec, J., and Jegelka, S.
\newblock How powerful are graph neural networks?
\newblock \emph{ArXiv}, abs/1810.00826, 2018.

\bibitem[Yan et~al.(2023)Yan, Zhou, Gao, Tang, and
  Zhang]{EfficientlyCoutingSubstructures}
Yan, Z., Zhou, J., Gao, L., Tang, Z., and Zhang, M.
\newblock Efficiently counting substructures by subgraph gnns without running
  gnn on subgraphs.
\newblock \emph{ArXiv}, abs/2303.10576, 2023.

\bibitem[You et~al.(2021)You, Gomes-Selman, Ying, and Leskovec]{IDAwareGNN}
You, J., Gomes-Selman, J.~M., Ying, R., and Leskovec, J.
\newblock Identity-aware graph neural networks.
\newblock In \emph{AAAI Conference on Artificial Intelligence}, 2021.

\bibitem[Zaheer et~al.(2017)Zaheer, Kottur, Ravanbakhsh, P{\'{o}}czos,
  Salakhutdinov, and Smola]{deepset}
Zaheer, M., Kottur, S., Ravanbakhsh, S., P{\'{o}}czos, B., Salakhutdinov, R.,
  and Smola, A.~J.
\newblock Deep sets.
\newblock In \emph{Advances in Neural Information Processing Systems}, pp.\
  3391--3401, 2017.

\bibitem[Zhang et~al.(2023)Zhang, Feng, Du, He, and Wang]{AComplteHierarchy}
Zhang, B., Feng, G., Du, Y., He, D., and Wang, L.
\newblock A complete expressiveness hierarchy for subgraph gnns via subgraph
  weisfeiler-lehman tests.
\newblock \emph{ArXiv}, abs/2302.07090, 2023.

\bibitem[Zhang \& Li(2021)Zhang and Li]{NGNN}
Zhang, M. and Li, P.
\newblock Nested graph neural networks.
\newblock \emph{ArXiv}, abs/2110.13197, 2021.

\bibitem[Zhao et~al.(2021)Zhao, Jin, Akoglu, and Shah]{GNNAK}
Zhao, L., Jin, W., Akoglu, L., and Shah, N.
\newblock From stars to subgraphs: Uplifting any gnn with local structure
  awareness.
\newblock \emph{ArXiv}, abs/2110.03753, 2021.

\bibitem[Zhou et~al.(2018)Zhou, Cui, Zhang, Yang, Liu, and
  Sun]{Zhou2018GraphNN}
Zhou, J., Cui, G., Zhang, Z., Yang, C., Liu, Z., and Sun, M.
\newblock Graph neural networks: A review of methods and applications.
\newblock \emph{ArXiv}, abs/1812.08434, 2018.

\end{thebibliography}
\bibliographystyle{icml2023}

\newpage
\appendix
\onecolumn

\section{$k$-FWL}\label{app:kwl}

In our main text, we mainly use the family of Weisfeiler-Lehman (WL) algorithm. Note that there's another form of graph isomorphic algorithm called Folklore Weisfeiler-Lehman (FWL) test, and the difference between $k$-WL and $k$-FWL lies in their tuple color update process. Concretely, at $t$-th iteration,
\begin{equation}
    c_k^{t}(\vv, G)=\text{Hash}\Big(c_k^{t-1}(\vv, G), \bmset{(
    c_k^{t-1}(\psi_i(\vv, u), G)|i\in [k])|u\in V(G)}\Big), 
\end{equation}

In this paper, we will strictly distinguish $k$-WL and $k$-FWL. A well known result is $\forall k\geq 1$, $k$-FWL is equivalent to $k+1$-WL, and we will show that $\forall l, \forall k\geq 1$, $k,l$-FWL is equivalent to $k+1,l$-WL in our proof. For simplicity, we will use $k,l$-FWL in most of our contexts and theorems.

\section{Localized $k,l$-WL: enhancement by locality} \label{sectionlocality}


Recently, various subgraph-enhanced GNNs have been proved to lift up expressivity of GNNs. However, it has been proved that the expressivity of traditional node-based subgraph-GNNs are upper-bounded by 3-WL, see \citep{UnderstandingSymmetry}. In this section, we put forward a localized variation of $k,l$-WL by considering locality, which is more computational efficient and is able to incorporate more existing subgraph GNNs. Concretely, we first select a subgraph according to some permutation-invariant strategies, then select labeled $l$-tuples within the subgraph. Further, in every single run of $k$-WL with $l$ labels, a localized update procudure is performed on the selected subgraph instead of on the whole graph. Note that different subgraphs do not need to have the same sizes as long as the subgraph selection policy is permutation invariant. While a localized $k,l$-WL obviously has less computation cost than full $k,l$-WL, it is even more expressive in some situations, see \cref{example} for an intuitive example.



While all three parameters, dimension of WL algorithm $k$, number of IDs $l$, and subgraph size $m$ will affect graph isomorphism power of the algorithm, see \cref{example} for illustration. However, complete theoretical analysis of localized $k,l$-WL is non-trivial, which we leave for future work.

While localization affects the expressive power of $k,l$-WL, a comprehensive theoretical analysis is non-trivial, which we leave for future work. Here we only give a preliminary discussion on the relationship between parameters $k$, $l$ and localization. We prove in this paper that for $k\geq 2$, $k+1,l$-WL is strictly more powerful than $k,l+1$-WL (see \cref{k-1l+1<=kl}), though at same magnitude of complexity. This implies that $k$ is more important than $l$. Similarly, effects of localization can be simulated by parameter $l$ to some extent. For example, $t$ rounds of message passing starting from the labeled root node results to a $t$-hop subgraph. This observation indicates that localization is generally less important and is likely to be simulated by labeling. An interesting phenomenon is that increasing subgraph size does not necessarily increase expressivity. However, the concrete analysis is complicated, and  we leave complete theoretical analysis of the impact of localization on expressivity as an open question.

After introducing the localization, there are also some additional complexity reduction methods to be discussed. We mainly focus on selecting subgraph (so called 'message passing' range) to reduce number of labeled tuples and decrease duplicated computation. While the labeled nodes should definitely be included in the selected subgraph, a localized update is able to accelerate the algorithm.

Some of common and useful subgraph selection policies applicable to localized $k,l$-WL include: 
\begin{itemize}
    \item Label based selection: we only perform $k,l$-WL on the labeled nodes, while unlabeled nodes are ignored for each labeled $l$-tuple. This subgraph  selection is useful for counting substructure.
    \item Full graph: we directly run $k,l$-WL on the original whole graph. This is the original form of $k,l$-WL with no complexity reduction. 
    \item Node based selection: for example, extract K-hop ego-net of each labeled node. Node based selection recovers a wide range of subgraph GNNs, including K-hop GNN.
\end{itemize}

\section{Algorithm of localized $k,l$-WL}\label{algorithmpresentappendix}

We formally use pseudo code to describe our localized $k,l$-WL in \cref{klmWL}.

\begin{algorithm}[tb]
   \caption{Localized $k,l$-WL}
   \label{klmWL}
\begin{algorithmic}
   \STATE {\bfseries Input:} An uncolored, undirected. unlabeled graph $G=(V,E)$, subgraph generation policy $\mathcal P_s$, label node tuple generation policy $\mathcal P_t$,  dimension of WL test $k$, number of labels $l$, subgraph selection policy (optional, align with $\mathcal P_s$)\\
   \STATE
   \STATE Generate unordered subgraph $\{G_i\}$ according to $\mathcal P_s$.
   \FOR{$G_i\in \{G_i\}$}
   \STATE Generate labeled tuple $(v_{g1},...,v_{gl})\in G_i$
   \FOR{ordered graph $G^g$}
   \STATE \textbf{Initialization}: $C[\vec v]=_0\chi^k_{G^g}(\vec v)$, for all $\vec v\in V^k$;$M=\emptyset$;$List =\{_0\chi^k_{G^g}(\vec v):\vec v\in V(G^g)^k\}$;$S=\emptyset$. 
    \STATE ** All tuples initially colored with their isomorphism types and labels $L(\vec v)$;multiset $M$ empty; work list initialized with the initial color-label classes of $k$-tuples in order subgraph $G^g$; multiset $S$ storing color-labels of ordered subgraphs emtpy.
    \WHILE{$L\neq \empty$}
    \FOR{each tuple color-label class $c=(C,L)\in List$}
    \STATE remove $c$ from $List$
    \FOR{each tuple $\vec w$ with $C[\vec w]=c$}
    \FOR{each $j\leq k$}
    \FOR{each $u\in V(G^g)$}
    \STATE let $\vec v=\vec w[j,u]$;
    \STATE add ($\vec v,(C,L)[\vec v[1,u]], ..., (C,L)[\vec v[k,u]]$) to M.
    \ENDFOR
    \ENDFOR
    \ENDFOR
    \ENDFOR
    \STATE Perform Radix Sort of $M$.
    \STATE Scan $M$ replacing tuples $(\vec v, c_1^1, ...,c_k^1),...,(\vec v, c_1^r, ...,c_k^r)$ with the single tuples $(C[\vec v];(\vec v, c_1^1,...,c_k^1),...,(\vec v, c_1^r,...,c_k^r);\vec v)$.
    \STATE Perform Radix Sort of $M$.
    Scan $M$ for each color-label class $c=(C,L)$ that has been split, leave the larget part still color-labeled $c$, and update the colors $C$ of other parts, remain labels $L$ unchanged; add new color-label types to $List$.
    \ENDWHILE
    \STATE Append the current color-labeling of all $k$-tuples in the labeled subgraph $V(G^g)^k$ to $S$.
    \ENDFOR
    \ENDFOR
    \STATE Output multiset $S$ that contains color-labels for $k$-tuples in all ordered subgraphs.
\end{algorithmic}
\end{algorithm}

\section{Mathematical form of $1,l$-WL}

Here, we will give an example of tensor implementation for $1,l$-WL. Each node $v$ is represented using a $l$-dimensional vector $\mathbf{h}(v)$, and the $j^{th}$ component of the vector $\mathbf{h}_j(v)$ corresponds to label $j(j=1,2,...,l)$. To keep the representation permutation invariant, we need to adopt a permutation invariant labeling method, for instance, pool over a permutation group of IDs on the k-tuple. Denote set $M=(1,2,...,l)$, we can use the symmetric group $Sym(M)$ (i.e. the full permutation with $k!$ elements) as the permutation group acting on the ID set. The group is not unique as long as it makes the graph level representation permutation invariant, but to facilitate understanding, we will select $Sym(M)$ at first in our illustration. For each element $g$ of the permutation group, the labeled graph $X^g$ is represented by a matrix shaped $(n, k)$: $\mathbf{H}(X^{g})=(\mathbf{h}^{g}(v_1),...,\mathbf{h}^{g}(v_n))^T$. Since there are $k!$ elements in $S_k$, thus the representation for the graph $X$ will be a $(l!, n, l)$ shaped tensor: $\mathbf{H}(X)=(\mathbf{H}(X^{g_1}), ..., \mathbf{H}(X^{g_{l!})}))^T$. 

At initialization, in each labeled subgraph $G^{g}$, node $v_i$ with label $j$ is initialized as $\mathbf{x}^g_0(v_i)=(\mathbbm{1}(j=1),...,\mathbbm{1}(j=l))$, i.e. the $j^{th}$ element is $1$ and other elements are $0$. The vector will be $\mathbf{0}$ if the node is unlabeled. Thus the $H_0^{g}(X)$ is a sparse matrix with $l$ elements of $1$.

In each iteration $t$ of the 1-WL on k-tuple algorithm, the update process is performed as below. First, we aggregate information from neighbors of each node in each subgraph and update the representation of nodes. One instance of aggregation can be described as follows:

\begin{equation} \label{aggregate}
    \mathbf{\bar h}_{t+1}^{g}(v_i)=U_t\bigg (\mathbf{h}_{t}^{g}(v),\sum_{j:v_j\in \mathcal N(v_i)}M_t\Big (\mathbf{h}_{t}^{g}(v_j)\Big )\bigg )
\end{equation}

As a simple instance, the message function $M_t$ can be identity mapping, and the node update function $U_t$ simply adds two terms (vector in the last iteration, and aggregated neighbors' information). The summation can also be replaced by other permutation invariant pooling methods, e.g. mean or max.

Note that we then apply hash functions on three dimensions respectively to get final representation of iteration $t+1$. On the node level in each permutation, the hash function $f_{node}$ is applied on vector (tuple) $\mathbf{x}^{(m)}(v)$, so there are $2^l$ different output hash values. Then we apply hash function $f_{lbg}$ on each labeled graph $G^m$, in which we regard the multi-set (instead of tuple) of nodes as in 1-WL. Finally, hash function $f_{graph}$ is applied to the graph level, in which the multi-set of labeled subgraphs will be considered. 

\begin{equation} \label{node_hash}
    \mathbf{h}_{t+1}^{g}(v_i)=f_{node}(\mathbf{\bar h}_{t+1}^{g}(v_i)), \forall g\in S_l, \forall i\in (1,2,...,n)
\end{equation}
\begin{equation}  \label{lbg_hash}
    \mathbf{H}_{t+1}^{g}=f_{lbg}(\{\mathbf{h}_{t+1}^{g}(v_i)|i\in (1,2,...,n)\}), \forall g\in S_l
\end{equation}
\begin{equation}  \label{graph_hash}
    \mathbf{H}_{t+1}(X)=f_{graph}(\{\mathbf{X}_{t+1}^{g}|g\in S_l\})
\end{equation}

where $\{\cdot\}$ means a multi-set. Actually, in the hash step, each latter equation can use the tensors in the previous equation either before or after hashing.

\section{Proof}\label{app:proof}

\subsection{Notations}
\paragraph{Isomorphism type of node tuple} $k,l$-WL and $k$-WL use the isomorphism type of tuple to initialize colors, which is defined as follows:

Given graphs $G^1=(V^1, E^1, X^1),G^2=(V^2,E^2,X^2)$ and $k$-tuples $\vv^1,\vv^2$ in $G^1, G^2$ respectively. $\vv_1, \vv_2$ have the same isomorphism type iff
\begin{enumerate}
    \item $\forall i_1,i_2\in [k]$,  $\vv^1_{i_1}=\vv^1_{i_2}\leftrightarrow\bold \vv^2_{i_1}=\vv^2_{i_2}$.
    \item $\forall i\in [k], X^1_{\vv^1_i}=X^2_{\vv^2_i}$.
    \item $\forall i_1,i_2\in [k]$, $(\vv^1_{i_1},\vv^1_{i_2})\in E_1\leftrightarrow (\vv^2_{i_1},\vv^2_{i_2})\in E_2$.
\end{enumerate} 

\paragraph{Expressivity comparison}

Given two function $f, g$, $f$ can be expressed by $g$ means that there exists a function $\phi$ $\phi\circ g=f$, which is equivalent to given arbitrary input $H, G$, $f(H)=f(G)\Rightarrow g(H)=g(G)$. We use $f\to g$ to denote that $f$ can be expressed with $g$. If both $f\to g$ and $g\to f$, there exists a bijective mapping between the output of $f$ to the output of $g$, denoted as $f\leftrightarrow g$.

Here are some basic rule.
\begin{itemize}
\item $g\to h\Rightarrow f\circ g\to f\circ h$.
\item $g\to h, f\to s\Rightarrow f\circ g\to s\circ h$.
\item $f$ is bijective, $f\circ g\to g$
\end{itemize}

\paragraph{$k,l$-WL} In this section, we use $c_{k}^t$ to denote the color produced by $k$-WL in $t$ iteration and $c_{k}$ to denote the stable color. 

Let $G^{\vv}=(V, E, X^{\vv})$ denote the graph labeled by tuple $\vv$. The only difference from the original graph is that node feature $X$ add extra label. $X^{\vv}_u=\tuple{X_u,\mset{i|\vv_i=u,i\in [l]}}$.

Let $\psi_{\va}(\vv, \vu)$ denote a tuple produced by replacing $\vv_{\va_i}$ with $\vu_i$ for $i=1,2,...,|\va|$, where $\va$ is an index tuple, $i<j\Rightarrow \va_i<\va_j$. Given $\vv\in V^k, \va\in [k]^l$, $\vv_{\va}$ denote a $l$-tuple whose $i$-th element is $\vv_{\va_i}$.

We need to reduce colors $c^{t}(\vu, G^{\vv})$ of $|V|^{l+k}$ tuples $\vv,\vu$ to a single color $c^{t}(G)$ of the full graph. Two pooling methods exist.
    \begin{itemize}
        \item Tuple-label (TL) pooling: First aggregate tuple representations of the same labeled graph, then aggregate representations of different labeled graphs. It is also the pooling method used in maintext.
        \begin{equation}
            c^{t}(G)=\text{Hash}(\mset{\text{Hash}(
            \mset{c^{(t)}(\vu, G^{\vv})|\vu \in V^k})|\vv\in V^l}).
        \end{equation}
        \item Label-tuple (LT) pooling: First aggregate representations of the same tuple in different labeled graphs, then aggregate tuple representations.
        \begin{equation}
            c^{t}(G)=\text{Hash}(\mset{\text{Hash}(\mset{c^{(t)}(\vu, G^{\vv})|\vv \in V^l})|\vu\in V^k}).
        \end{equation}
    \end{itemize}
    Let $k,l$-WL($LT$) and $k,l$-WL($TL$) denote $k,l$-WL with the two pooling methods. By default, $k,l$-WL use $TL$-pooling, which is more similar to relational pooling. However, we prove that $LT$ pooling leads to higher expressivity. 
    
    If the number of iteration $t$ is specified, $t$ is by default large enough to produce stable color: Given two graphs $\forall \vu_1, \vu_2\in V^k, \vv_1, \vv_2 \in V^l,$
    \begin{small}
        \begin{equation}
        \!\!c^{t}\!(\vu_1,\!G^{\vv_1}\!)\!=\!c^{t}\!(\vu_2,\!G^{\vv_2}\!)\!\leftrightarrow\!  c^{t+1}\!(\vu_1,\!G^{\vv_1}\!)=c^{t+1}\!(\vu_2,\!G^{\vv_2}\!)
    \end{equation}
    \end{small}

Tuple colors of $k,l$-WL is $c_{k}$ on labeled graph. We use $c_{k,l,TL},c_{k,l,LT}$ to denote graph color produced by $k,l$-WL with TL and LT pooling.

\subsection{Connection to existing hierarchy}
\begin{proposition}
    Given two algorithms $A$ $B$, which produce color $c_A(G), c_B(G)$ with graph $G$ as input. If $A\cong B$, then forall graph $G=(V,E,X), H=(V',E',X')$, 
    \begin{equation}
        \mset{c_A(G^\vv)|\vv\in V^l}=\mset{c_A(H^\vv)|\vv\in {V'}^l}\Leftrightarrow \mset{c_A(G^\vv)|\vv\in V^l}=\mset{c_A(H^\vv)|\vv\in {V'}^l}
    \end{equation} 
\end{proposition}

Therefore, as $1-WL\cong 2-WL, k+1-WL\cong k-FWL$, $1,l-WL\cong 2,l-WL, k,l-FWL\cong k+1,l-WL$.

\subsection{expressivity hierarchy}\label{SectionProofExpressiveHierarchy}

\begin{lemma}\label{lem:kwl-pool}
Given a graph $G=(V, E, X)$, $\forall k\ge 2$, $\forall t\ge 0, 0<h\le k, \forall \vv\in V^k, \va\in [k]^h $ with no duplicated elements, $c_k^{t+h}(\vv, G)\to \mset{c_k^{t}(\psi_{\va}(\vv, \vu), G)|\vu\in V^h}$
\end{lemma}

\begin{proof}
We enumerate $h$,

$h=1$: 
$\forall i\in k$
\begin{equation}
c_k^{t+1}(\vv, G)=\text{Hash}(c_k^{t}(\vv, G), \big(\mset{c_k^{t}(\psi_{j}(\vv, u), G)|u\in V}|j\in [k]\big)\to \mset{c_k^{t}(\psi_{i}(\vv, u), G)|u\in V}
\end{equation}

Assuming that $h>1$, $\forall \vv\in V^k, \va\in [k]^{h-1} $ with no duplicated elements, $c_k^{t+h-1}(\vv, G)\to \mset{c_k^{t}(\psi_{\va}(\vv, \vu), G)|\vu\in V^{h-1}}$.

$\forall \va\in [k]^{h}$
\begin{align}
    c_k^{t+h}(\vv, G)&\to \mset{c_k^{t+h-1}(\psi_{\va_h}(\vv, u), G)|u\in V}\\
    &\to \mset{\mset{c_k^{t}(\psi_{\va_{:h}}(\psi_{\va_h}(\vv, u), \vu), G)|\vu\in V^{h-1}}|u\in V}\\
    &\to \mset{c_k^{t}(\psi_{\va}(\vv, \vu), G)|\vu\in V^{h}}
\end{align}
\end{proof}

\begin{lemma}\label{lem:kl<=kl+1}
    $\forall k\ge 2, l\ge 0$, $k,l$-WL(TL) $\preceq$ $k,l+1$-WL(TL), $k,l$-WL(LT) $\preceq$ $k,l+1$-WL(LT)
\end{lemma}
\begin{proof}
Given graph $G=(V,E,X)$, we first prove that $\forall \vv\in V^k, \vu\in V^{l+1}, c_{k}^t(\vv, G^{\vu})\to c_{k}^t(\vv, G^{\vu_{:l+1}})$ by enumerating $t$.

\begin{enumerate}
    \item $t=0$, the color is isomorphism type.
    $\forall G_1=(V_1,E_1,X_1), G_2=(V_2, E_2, X_2),\forall \vv^1\in V_1^{k}, \vv^2\in V_2^{k}, \vu^1\in V_1^{l+1}, \vu^2\in V_2^{l+1}$,
    \begin{align}
        c_{k+1}^0(\vv^1, G_1^{\vu^1})=c_{k+1}^0(\vv^2, G_2^{\vu^2})\Rightarrow (\forall i_1,i_2\in [k], \vv^1_{i_1}=\vv^1_{i_2}\leftrightarrow \vv^2_{i_1}=\vv^2_{i_2})
        \\
        \land (\forall i\in [k], {X_1^{\vu^1}}_{\vv^1_i}={X_2^{\vu^2}}_{\vv^2_i})
        \land (\forall i_1,i_2\in [k], (\vv^1_{i_1},\vv^1_{i_2})\in E_1\leftrightarrow (\vv^2_{i_1},\vv^2_{i_2})\in E_2)\\
        \Rightarrow (\forall i_1,i_2\in [k+1], \vv^1_{i_1}=\vv^1_{i_2}\leftrightarrow \vv^2_{i_1}=\vv^2_{i_2})
        \\
        \land (\forall i\in [k], {X_1^{\vu^1_{:l+1}}}_{\vv^1_i}={X_2^{\vu^2_{:l+1}}}_{\vv^2_i})
        \land (\forall i_1,i_2\in [k], (\vv^1_{i_1},\vv^1_{i_2})\in E_1\leftrightarrow (\vv^2_{i_1},\vv^2_{i_2})\in E_2)\\
        \Rightarrow c_{k}^0(\vv^1, G_1^{\vu^1_{:l+1}})=c_{k}^0(\vv^2, G_2^{\vu^2_{:l+1}})
    \end{align}
    Therefore, $c_{k}^0(\vv, G^{\vu})\to c_{k}^0(\vv, G^{\vu_{:l+1}})$
\item $\forall t>0$,
\begin{align}
    c_{k}^t(\vv, G^{\vu})=\text{Hash}\big(c_{k}^{t-1}(\vv, G^{\vu}), (\mset{c_{k}^{t-1}(\psi_i(\vv, u), G^{\vu})|u\in V}|i\in [k])\big)\\
    \to \text{Hash}\big(c_{k}^{t-1}(\vv, G^{\vu_{:l+1}}), (\mset{c_{k}^{t-1}(\psi_i(\vv, u), G^{\vu_{:l+1}})|u\in V}|i\in [k])\big)\\
    \to c_{k}^t(\vv, G^{\vu_{:l+1}}).
\end{align}
\end{enumerate}

With TL pooling:
\begin{align}
    c_{k,l+1, TL}^t(G)=\text{Hash}(\mset{\mset{c_{k}^t(\vv, G^{\vu})|\vv\in V^k}|\vu\in V^{l+1}})\\
    \to \text{Hash}(\mset{\mset{c_{k}^t(\vv, G^{\vu_{:l+1}})|\vv\in V^k}|\vu\in V^{l+1}})\\
    \to \text{Hash}(\mset{\mset{c_{k}^t(\vv, G^{\vu})|\vv\in V^k}|\vu\in V^{l}})\\
    \to c_{k, l, TL}^t(G)
\end{align}

With LT pooling:

\begin{align}
    c_{k,l+1, LT}^t(G)=\text{Hash}(\mset{\mset{c_{k}^t(\vv, G^{\vu})|\vu\in V^{l+1}}|\vv\in V^{k}})\\
    \to \text{Hash}(\mset{\mset{c_{k}^t(\vv, G^{\vu_{:l+1}})|\vu\in V^{l+1}}|\vv\in V^{k}})\\
    \to \text{Hash}(\mset{\mset{c_{k}^t(\vv, G^{\vu})|\vu\in V^{l}}|\vv\in V^k})\\
    \to c_{k, l, LT}^t(G)
\end{align}
\end{proof}

\begin{lemma}\label{lem:kl<=k+1l}
    $\forall k\ge 2, l\ge 0$, $k,l$-WL(TL) $\preceq$ $k+1,l$-WL(TL), $k,l$-WL(LT) $\preceq$ $k+1,l$-WL(LT)
\end{lemma}
\begin{proof}
Given graph $G=(V,E,X)$, we first prove that $\forall \vv\in V^{k+1}, \vu\in V^{l}, c_{k+1}^t(\vv, G^{\vu})\to c_{k}^t(\vv_{:k+1}, G^{\vu})$ by enumerating $t$.

\begin{enumerate}
    \item $t=0$, the color is isomorphism type.
    $\forall G_1=(V_1,E_1,X_1), G_2=(V_2, E_2, X_2),\forall \vv^1\in V_1^{k+1}, \vv^2\in V_2^{k+1}, \vu^1\in V_1^l, \vu^2\in V_2^l$,
    \begin{align}
        c_{k+1}^0(\vv^1, G_1^{\vu^1})=c_{k+1}^0(\vv^2, G_2^{\vu^2})\Rightarrow (\forall i_1,i_2\in [k+1], \vv^1_{i_1}=\vv^1_{i_2}\leftrightarrow \vv^2_{i_1}=\vv^2_{i_2})
        \\
        \land (\forall i\in [k+1], {X_1^{\vu^1}}_{\vv^1_i}={X_2^{\vu^2}}_{\vv^2_i})
        \land (\forall i_1,i_2\in [k+1], (\vv^1_{i_1},\vv^1_{i_2})\in E_1\leftrightarrow (\vv^2_{i_1},\vv^2_{i_2})\in E_2)\\
        \Rightarrow (\forall i_1,i_2\in [k+1], \vv^1_{i_1}=\vv^1_{i_2}\leftrightarrow \vv^2_{i_1}=\vv^2_{i_2})
        \\
        \land (\forall i\in [k], {X_1^{\vu^1}}_{\vv^1_i}={X_2^{\vu^2}}_{\vv^2_i})
        \land (\forall i_1,i_2\in [k], (\vv^1_{i_1},\vv^1_{i_2})\in E_1\leftrightarrow (\vv^2_{i_1},\vv^2_{i_2})\in E_2)\\
        \Rightarrow c_{k}^0(\vv^1_{:k+1}, G_1^{\vu^1})=c_{k}^0(\vv^2_{:k+1}, G_2^{\vu^2})
    \end{align}
    Therefore, $c_{k}^0(\vv, G^{\vu})\to c_{k}^0(\vv, G^{\vu_{:l+1}})$
\item $\forall t>0$,
\begin{align}
    c_{k+1}^t(\vv, G^{\vu})=\text{Hash}\big(c_{k+1}^{t-1}(\vv, G^{\vu}), (\mset{c_{k+1}^{t-1}(\psi_i(\vv, u), G^{\vu})|u\in V}|i\in [k+1])\big)\\
    \to \text{Hash}\big(c_{k+1}^{t-1}(\vv, G^{\vu}), (\mset{c_{k+1}^{t-1}(\psi_i(\vv, u), G^{\vu})|u\in V}|i\in [k])\big)\\
    \to \text{Hash}\big(c_{k}^{t-1}(\vv_{:k+1}, G^{\vu}), (\mset{c_{k}^{t-1}(\psi_i(\vv_{:k+1}, u), G^{\vu})|u\in V}|i\in [k])\big)\\
    \to c_{k}^t(\vv_{k+1}, G^{\vu}).
\end{align}
\end{enumerate}

With TL pooling:
\begin{align}
    c_{k+1,l,TL}^t(G)=\text{Hash}(\mset{\mset{c_{k+1}^t(\vv, G^{\vu})|\vv\in V^{k+1}}|\vu\in V^{l}})\\
    \to \text{Hash}(\mset{\mset{c_{k}^t(\vv, G^{\vu})|\vv\in V^{k+1}}|\vu\in V^{l}})\\
    \to \text{Hash}(\mset{\mset{c_{k}^t(\vv, G^{\vu})|\vv\in V^k}|\vu\in V^{l}})\\
    \to c_{k, l,TL}^t(G)
\end{align}

With LT pooling:

\begin{align}
    c_{k+1,l,LT}^t(G)=\text{Hash}(\mset{\mset{c_{k}^t(\vv, G^{\vu})|\vu\in V^{l}}|\vv\in V^{k+1}})\\
    \to \text{Hash}(\mset{\mset{c_{k}^t(\vv_{:k+1}, G^{\vu})|\vu\in V^{l}}|\vv\in V^{k+1}})\\
    \to \text{Hash}(\mset{\mset{c_{k}^t(\vv, G^{\vu})|\vu\in V^{l}}|\vv\in V^k})\\
    \to c_{k, l, LT}^t(G)
\end{align}
\end{proof}

\begin{lemma}\label{lem:kl+1<=k+1l}
    $\forall k\ge 2, l\ge 0$, $k,l+1$-WL(TL) $\preceq$ $k+1,l$-WL(TL)
\end{lemma}
\begin{proof}
Given graph $G=(V,E,X)$, we first prove that $\forall \vv\in V^{k}, \vu\in V^{l}, w\in V, c_{k+1}^t(\vv|\!|w, G^{\vu})\to c_{k}^t(\vv, G^{\vu|\!|w})$ by enumerating $t$.

\begin{enumerate}
    \item $t=0$, the color is isomorphism type. 
    $\forall G_1=(V_1,E_1,X_1), G_2=(V_2, E_2, X_2),\forall \vv^1\in V_1^{k}, \vv^2\in V_2^{k},w^1\in V_1, w_2\in V_2, \vu^1\in V_1^l, \vu^2\in V_2^l$,
    \begin{align}
        c_{k+1}^0(\vv^1|\!|w^1, G_1^{\vu^1})=c_{k+1}^0(\vv^2|\!|w^2, G_2^{\vu^2})\Rightarrow 
        \\
        \big(\forall i_1,i_2\in [k+1], (\vv^1|\!|w^1)_{i_1}=(\vv^1|\!|w^1)_{i_2}\leftrightarrow (\vv^2|\!|w^2)_{i_1}=(\vv^2|\!|w^2)_{i_2}\big)
        \\
        \land (\forall i\in [k], {X_1^{\vu^1}}_{\vv^1_i}={X_2^{\vu^2}}_{\vv^2_i})\land ({X_1^{\vu^1}}_{w^1}={X_2^{\vu^2}}_{w^2})
        \\
        \land (\forall i_1,i_2\in [k+1], ((\vv^1|\!|w^1)_{i_1},\vv^1_{i_2})\in E_1\leftrightarrow ((\vv^2|\!|w^2)_{i_1},(\vv^2|\!|w^2)_{i_2})\in E_2)\\
        \Rightarrow (\forall i_1,i_2\in [k], \vv^1_{i_1}=\vv^1_{i_2}\leftrightarrow \vv^2_{i_1}=\vv^2_{i_2})
        \\
        \land (\forall i\in [k], {X_1^{\vu^1|\!|w^1}}_{\vv^1_i}={X_2^{\vu^2|\!|w^2}}_{\vv^2_i})
        \land (\forall i_1,i_2\in [k], (\vv^1_{i_1},\vv^1_{i_2})\in E_1\leftrightarrow (\vv^2_{i_1},\vv^2_{i_2})\in E_2)\\
        \Rightarrow c_{k}^0(\vv^1, G_1^{\vu^1|\!|w^1})=c_{k}^0(\vv^2, G_2^{\vu^2|\!|w^2})
    \end{align}
    Therefore, $c_{k+1}^t(\vv||w, G^\vu)\to c_{k}^t(\vv, G^{\vu|\!|w})$.
\item $\forall t>0$,
\begin{align}
    c_{k+1}^t(\vv|\!|w, G^{\vu})=\text{Hash}\big(c_{k+1}^{t-1}(\vv|\!|w, G^{\vu}), (\mset{c_{k+1}^{t-1}(\psi_i(\vv|\!|w, u), G^{\vu})|u\in V}|i\in [k+1])\big)\\
    \to \text{Hash}\big(c_{k+1}^{t-1}(\vv|\!|w, G^{\vu}), (\mset{c_{k+1}^{t-1}(\psi_i(\vv, u)|\!|w, G^{\vu})|u\in V}|i\in [k])\big)\\
    \to \text{Hash}\big(c_{k}^{t-1}(\vv, G^{\vu|\!|w}), (\mset{c_{k}^{t-1}(\psi_i(\vv, u), G^{\vu|\!|w})|u\in V}|i\in [k])\big)\\
    \to c_{k}^t(\vv, G^{\vu|\!|w}).
\end{align}
\end{enumerate}

With TL pooling:
\begin{align}
    c_{k+1,l, TL}^t(G)=\text{Hash}(\bmset{\mset{c_{k+1}^t(\vv|\!|w, G^{\vu})|\vv\in V^{k},w\in V}|\vu\in V^{l}})\\
\end{align}
According to Lemma~\ref{lem:kwl-pool},
\begin{align}
    c_{k+1,l, TL}^t(G)    \to \text{Hash}(\Bmset{\bmset{\mset{c_{k+1}^{t-k}(\vv'|\!|w, G^{\vu})|\vv'\in V^{k}}|\vv\in V^{k},w\in V}|\vu\in V^{l}})\\
    \to \text{Hash}(\bmset{\mset{c_{k+1}^{t-k}(\vv'|\!|w, G^{\vu})|\vv'\in V^{k}}|w\in V,\vu\in V^{l}})\\
    \to \text{Hash}(\bmset{\mset{c_{k}^{t-k}(\vv', G^{\vu|\!|w})|\vv'\in V^{k}}|w\in V,\vu\in V^{l}})\\
    \to c_{k, l+1, TL}^{k-1}(G)
\end{align}
\end{proof}

\subsection{$\not\preceq$ Results}\label{sectionproofexample}

According to Lemma 27 from ~\citep{OSAN}, 
\begin{proposition}
    $\forall k\ge 1$, there exists a pair of graphs that $2, k$-WL(TL) can differentiate while $k+1$-WL cannot.
\end{proposition}
\begin{corollary}
    $\forall k\ge 2, l\ge 0$, $k,l$-WL(TL) $\prec$ $k, l+1$-WL(TL), $k,l$-WL(TL) $\prec$ $k+1, l$-WL(TL).
\end{corollary}
\begin{proof}
According to Lemma~\ref{lem:kl+1<=k+1l}, $2,k+l-1$-WL $\preceq 3, k+l-2$-WL $\preceq$ ... $\preceq$ $k, l+1$-WL $\preceq$ $k+1,l$-WL. Therefore, there exists a pair of graphs that $k,l+1$-WL and $k+1, l$-WL can differentiate while $k,l$-WL cannot.  Moreover, according to Lemma~\ref{lem:kl<=k+1l} and Lemma~\ref{lem:kl<=kl+1}, $k,l$-WL $\preceq$ $k,l+1$-WL and $k,l$-WL $\preceq$ $k+1, l$-WL. Therefore, $k,l$-WL $\prec$ $k, l+1$-WL, $k,l$-WL $\prec$ $k+1, l$-WL.
\end{proof}

Then we prove that $k+1,l$-WL(TL)$\not\preceq$ $k,l+1$-WL in the following paragraphs. 

\paragraph{Pebble Game}

\citet{PebbleGame} propose $C_k$ game as follows:

Given two graphs $G, H$, and $k$ pairs of pebbles $x_1,x_2,...,x_k$. Initially, no pebbles is placed on the graph. Two players act as follows in one epoch.
\begin{enumerate}
    \item Player 1 picks up the $x_i$ pebble pair for some $i$. 
\item Player 1 chooses a graph $F$ from $\{G, H\}$. 
\item Player 1 chooses a set $A$ of vertices from $F$. Player 2 answers with a set $B$ of vertices from the other graph. $|B|=|A|$
\item Player 1 places one of the $x_i$ pebbles on some vertex $b \in B$. Player 2 answers by placing the other pebble on some $a \in A$.
\item If the subgraph induced by pebbles of Player 1 is not isomorphism to that of Player 2, player 1 wins. Otherwise, player 2 wins. 
\end{enumerate}

\citet{PebbleGame} also prove that 
\begin{proposition}
    Player 2 has an winning strategy in $C_k$ game on two graphs $G, H$ iff $c_k(G)=c_k(H)$
\end{proposition}

\paragraph{CFI graph}

\begin{figure}
    \centering
    \includegraphics[width=\textwidth]{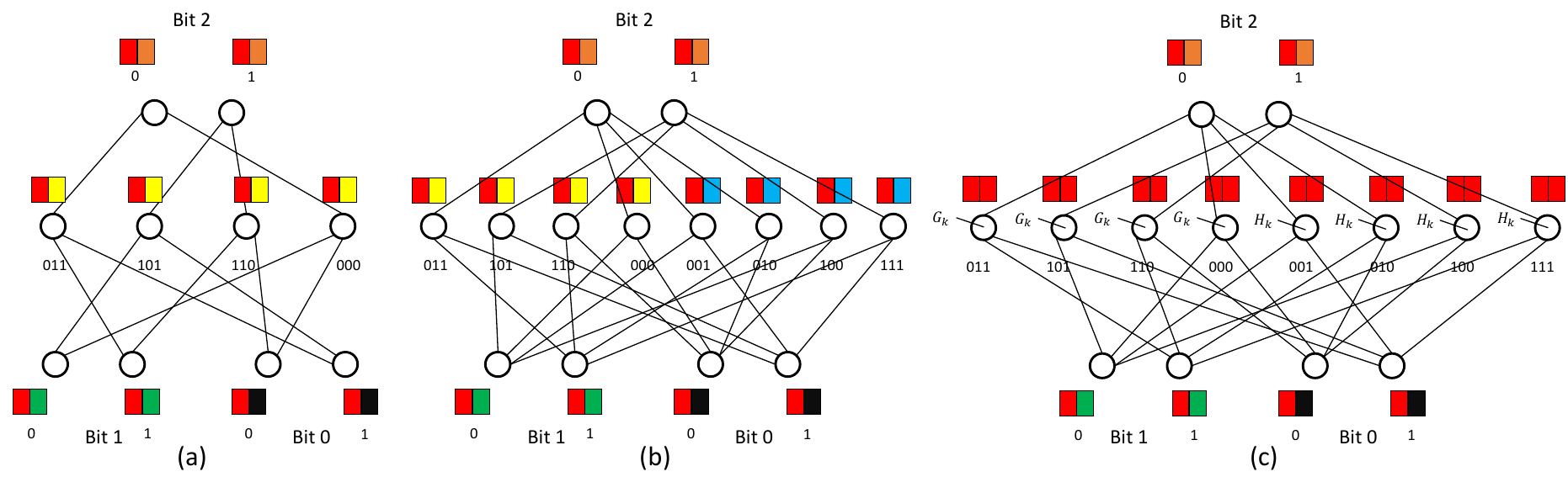}
    \caption{(a) $X_k$ }
    \label{fig:CFIblocks}
\end{figure}

\citet{PebbleGame} have built a family of graphs that $k+1$-WL can differentiate while $k$-WL cannot. We restate their conclusion here.

First, \citet{PebbleGame} define a kind of basic block $\chi_k$ as  follows:

\begin{enumerate}
    \item $2^{k-1}$ nodes, each representing a $k$-bit binary with an even number of 1’s. They all have the same color.
    \item  $2k$ nodes, representing $k$ binary bits $0/1$ $(a_i,b_i)$. Each bit has a different color. 
    \item Edges that connects a binary with each of its $k$ bits.
\end{enumerate}
A $\chi_3$ block is shown in Figure~\ref{fig:CFIblocks}(a). The block has the following property (Lemma 6.1 in ~\citep{PebbleGame}).

\begin{proposition}\label{prop:CFIblock}
There are exactly $2^{k-1}$ automorphisms of $\chi_k$. Each is determined by interchanging $a_i$ and $b_i$ for each $i$ in some subset $S$ of $\{1, . . . , n\}$ of even cardinality.
\end{proposition}

Let $G=(V,E,X)$ denote an undirected graph whose minimal node degree $\ge 2$. The graph $\chi(G)$ is defined as follows. 
\begin{enumerate}
    \item $\forall v\in  V$, we replace $v$ by a copy of $\chi_{d(v)}$, namely $\chi(v)$, where $d(v)$ is the node degree of $v$. Each node in $\chi(v)$ inherit the node feature $X_{v}$.
    \item $\forall (v,w)\in E(G)$, we associate one of the pairs $\{a_i, b_i\}$ from $X(v)$, call this pair $a(v,w)$ and $b(v,w)$. Then, we draw the edges $(a(u,v), a(v,u)), (b(u,v), b(v,u))$.
\end{enumerate}

The twist of $\chi(G)$ ($\tilde \chi(G)$) is produced by arbitrarily choosing one edge $(v,w)\in E(G)$ and twist it ( replace $(a(u,v), a(v,u)), (b(u,v), b(v,u))$ with $(a(u,v), b(v,u)), (b(u,v), a(v,u))$. Two graphs are not isomorphic (Lemma 6.2 in \citep{PebbleGame}).
\begin{proposition}\label{prop:twistCFI}
    Let $\hat \chi (G)$ be constructed like $\chi(G)$, but with exactly $t$ of its edges twisted. Then $\hat \chi (G)$ is isomorphic to $\chi(G)$ iff $t$ is even, and $\hat \chi (G)$ is isomorphic to $\tilde \chi (G)$ iff t is odd.
\end{proposition}

Such graphs is called CFI graph. \citet{PebbleGame} prove that some CFI graphs is $k$-WL indistinguishable. 
\begin{proposition}\label{prop:WLCFI}
Let $T$ be a graph such that every seperator of $T$ has at least $k+1$ vertices. Then $k$-WL cannot differentiate $X(T)$ and $\tilde X(T)$.
\end{proposition}

Note that Proposition~\ref{prop:WLCFI} and Proposition~\ref{prop:twistCFI} do not use any property of the basic block other than Proposition~\ref{prop:CFIblock}. Therefore, we can use other basic blocks to keep these these two conclusions.

Based on results above, \citet{WLgoSparse} build a family of graphs. Let $K_l$ denote $l$-cluster with nodes colored as $1,2,...,l$. Let $H_k=X(K_{k+1}),G_k=\tilde X(K_{k+1})$. A direct corollory of Proposition~\ref{prop:WLCFI} is 

\begin{corollary}
    $\forall k\ge 2$, $k$-WL cannot differentiate $H_k$, $G_k$.
\end{corollary}

Moreover, \citet{WLgoSparse} and \citet{OSAN} prove the following conclusion.

\begin{proposition}
    $k+1$-WL and $2,k-1$-WL(TL) can differentiate $H_k, G_k$
\end{proposition}

\paragraph{Extended CFI graph}

We propose the following basic block to build other family of graphs. $\omega_k$ are as follows:

\begin{enumerate}
    \item $2^{k-1}$ nodes, each representing a $k$-bit binary with an even number of 1’s. They all have the same color $c_{e}$.
    \item $2^{k-1}$ nodes, each representing a $k$-bit binary with an odd number of 1’s. They all have the same color $c_{o}$.
    \item  $2k$ nodes, representing $k$ binary bits $0/1$ $(a_i,b_i)$. Each bit has a different color. 
    \item Edges that connects a binary with each of its $k$ bits.
\end{enumerate}
Am $\omega_3$ block is shown in Figure~\ref{fig:CFIblocks}(b). The block has the following property as $\chi_k$ block.

\begin{proposition}
There are exactly $2^{k-1}$ automorphisms of $\omega^k$. Each is determined by interchanging $a_i$ and $b_i$ for each $i$ in some subset $S$ of $\{1, . . . , n\}$ of even cardinality.
\end{proposition}
\begin{proof}
Any automorphism $\phi$ of $X_k'$ must keeps node colors. Therefore, $\phi$ must map $C_{e}$ nodes to $C_{e}$ nodes, and $\{a_i,b_i\}$ to $\{a_i,b_i\}$, so $\phi$ is also an automorphism of $X_k$. 

Therefore, $\phi$ is determined by interchanging $a_i$ and $b_i$ for each $i$ in some subset $S$ of $\{1, . . . , n\}$ of even cardinality. Moreover, each of such interchanging is also an automorphism. 
\end{proof}

Similar to $\chi_k$, we can define CFI graph $\omega(G)$ and its twist $\tilde \omega(G)$. The following proposition still holds.

\begin{proposition}\label{prop:twistECFI}
    Let $\hat \omega (G)$ be constructed like $\omega(G)$, but with exactly $t$ of its edges twisted. Then $\hat \omega (G)$ is isomorphic to $\omega(G)$ iff $t$ is even, and $\hat \omega (G)$ is isomorphic to $\tilde \omega (G)$ iff t is odd.
\end{proposition}
\begin{proposition}\label{prop:WLECFI}
Let $T$ be a graph such that every seperator of $T$ has at least $k+1$ vertices. Then $k$-WL cannot differentiate $\omega(T)$ and $\tilde omega(T)$.
\end{proposition}
A direct corollary is 
\begin{proposition}
    k-WL cannot distinguish $\omega(K_{k+1}),\tilde \omega(K_{k+1})$.
\end{proposition}

Moreover, as algorithms can directly ignore nodes with $c_{o}$ color
\begin{proposition}
    $k+1$-WL and $2,k-1$-WL(TL) can distinguish $\omega(K_{k+1}),\tilde \omega(K_{k+1})$.
\end{proposition}

\paragraph{A pair of graph that $k+1,l$-WL can distinguish while $k, l+1$-WL cannot}

We propose the another kind of basic block similar to $\omega_k$. $\gamma_{a,b}$ are as follows:

\begin{enumerate}
    \item $2^{b-1}$ nodes, each representing a $k$-bit binary with an even number of 1’s. Each node is connected to all nodes of a instance of $G_a$, which has a fixed special color $c_{br}$ on each node.
    \item $2^{b-1}$ nodes, each representing a $k$-bit binary with an odd number of 1’s. They all have the same color $c_{o}$. Each node is connected to all nodes of a instance of $H_a$, which has a fixed special color $C_{br}$ on each node.
    \item  $2b$ nodes, representing $k$ binary bits $0/1$ $(a_i,b_i)$. Each bit has a different color. 
    \item Edges that connects a binary with each of its $k$ bits.
\end{enumerate}
$G_a, H_a$ is generally called branch in the block. A $\gamma_{3,k}$ block is shown in Figure~\ref{fig:CFIblocks}(b). Branches $G_a, H_a$ can be considered as $c_{o}$ and $c_{e}$ in $\omega_k$. 
Similar to $\omega_k$, we can define CFI graph $\gamma_{a}(G)$ and its twist $\tilde \gamma_{a}(G)$, which is produced by using block $\gamma_{k, a}$. The following proposition still holds.

\begin{proposition}\label{prop:twistEGCFI}
    Let $\hat \gamma_a (G)$ be constructed like $\gamma_a(G)$, but with exactly $t$ of its edges twisted. Then $\hat \gamma_a (G)$ is isomorphic to $\gamma_{a}(G)$ iff $t$ is even, and $\hat \gamma_{a} (G)$ is isomorphic to $\tilde \gamma_{a} (G)$ iff t is odd.
\end{proposition}

Then we come to our primary result.

\begin{theorem}
    $\forall k\ge 2, l\ge 0$, $k,l+1$-WL(TL) cannot distinguish $\gamma_{k}(K_{k+l+1})$ and $\tilde \gamma_{k}(K_{k+l+1})$, while $k+1, l$-WL(TL) can.
\end{theorem}

We prove it in two steps.

\begin{lemma}
    $\forall k\ge 2, l\ge 0$, $k,l+1$-WL(TL) cannot distinguish $\gamma_{k}(K_{k+l+1})$ and $\tilde \gamma_{k}(K_{k+l+1})$.
\end{lemma}

\begin{proof}
Let $G=\gamma_{k}(K_{k+l+1})=(V,E,X),H=\tilde \gamma_{k}(K_{k+l+1})=(V,E',X)$. As $k,l+1$-WL produce node colors as the multiset of labeled graphs' $k$-WL color. We are going to prove that
\begin{equation}
    \forall \vu \in V^{l+1}, c_{k}(G^{\vu})=c_{k}(H^{\vu})
\end{equation}
which is equivalent to that $\forall \vu\in V^{l+1}$, player 2 has an winning strategy in $C_k$ game on $G^{\vu}$, $H^{\vu}$.

As $l+1<k+l+1$, there exists at least one basic block with no node in $\vu$ (not labeled). For simplicity, let $\gamma(0)$ denote the block. Let $\{(a_i,b_i)|i\in [k+l]\}$ denote the bit nodes in the block. Let $\{d_{i}|i+1\in [2^{k+l}]\}$ denote the nodes corresponding to binary numbers in $\gamma(0)$. Let $e_{i}$ denote the set of nodes in the branch connected to $d_{i}$. 

No matter where the twist happens, the results graphs are isomorphic, so we can simply assume that the twist happen in $(a_1, b_1)$. As the subgraphs induced by $V-V(\gamma(0))$ are isomorphic, we can define the following function $\phi: V\to V\cup \{-1\}$.

\begin{equation}
\phi(v)
=\begin{cases}
    v & \text{if }v\in V-V(\gamma(0))\\
    b_1 & \text{if }v=a_1\\
    a_1 & \text{if }v=b_1\\
    b_i & \text{if }v\in \{b_i|i\in [k+l],i\neq 1\}\\
    a_i & \text{if }v\in \{a_i|i\in [k+l],i\neq 1\}\\
    d_{i\wedge1}& \text{if }v\in \{d_i|i+1\in [2^{k+1}]\}\\
    -1 & \text{otherwise}
\end{cases}
\end{equation}
where $\wedge$ means bitwise xor.

Let $\vu\in (V\cup \{-1\})^k$ denote the configure of the pebbles, $\vu_i$ is the node that $x_i$ pebble placed. $\vu_i=-1$ means $x_i$ pebble is currently not placed in the graph. As $G_k, H_k$ are $k$-WL indistinguishable, player 2 has a winning strategy of $C_k$ game on $G_k, H_k$, which is equivalent to that given $F\in \{G_k, H_k\}$, $A\subseteq V(F)$, feasible configuration $\vf$ in $F$ and configuration $\vf'$ in the other graph, player 2 can answer a nodeset $B_k(A, F, \vf, \vf')$. Moreover, when player 1 select $a\in B_k(A, F, \vf, \vf')$, player 2 can answer a node $b_k(A, F, \vf, \vf')$ to keep induced subgraph isomorphism.

Given a configuration $\vu$ and a node subset $U$, $\vu\cap U$ means the configuration on the induced subgraph, which considers pebbles put out side ths subgraph as $-1$. Let $\mu_i$ denote a isomorphism node mapping from branch $e_i$ to the origin graph $G_k$ or $H_k$ and $\mu_i^{-1}$ denote the inverse mapping. Given a node set $S$ or a configuration $\vu$,$\mu_i(S),\mu_i^{-1}(S),\mu_i(\vu), \mu_i^{-1}(\vu)$ means element-wise transformation ($-1$ is not transformed). Given a graph $F$ and a node $S$, let $F[S]$ denote the subgraph of $F$ induced by $S$.

Then player 2 can use the same strategy on branches as in $C_k$ games on $G_k, H_k$.  
\begin{enumerate}
    \item Given a graph $F$ and node set $A$ and configuration $\vf$ in $F$ and $\vf'$ in the other graph, player 2 answer the following set.
    \begin{align}
    \{\phi(v)|v\in A, \phi(v)\neq -1\}\cup\\ \bigcup_{i=0}^{2^{k+l}-1} \mu_{i\wedge1}^{-1} B_k\Big(\mu_i(A\cap e_i), F[e_i], \mu_i(\vf\cap e_i), \mu_{i\wedge1}(\vf'\cap e_{i\wedge1})\Big)
    \end{align}
    \item If player 1 select $a\in B$. Player 2 can select
    \begin{equation}
        \begin{cases}
        \phi^{-1}(a)&\text{if } a\notin \bigcup_{i=0}^{2^{k+l}-1} e_i\\
        \mu^{-1}_{i\wedge1}\Big(
        b_k\big(\mu_i(A\cap e_i), F[e_i], \mu_i(\vf\cap e_i), \mu_{i\wedge1}(\vf'\cap e_{i\wedge1})\big)
        \Big) &\text{if } a\in e_i
        \end{cases}
    \end{equation}
\end{enumerate}
\end{proof}

\begin{lemma}
    $k+1, l$-WL can distinguish $\gamma_{k}(K_{k+l+1})$ and $\tilde \gamma_{k}(K_{k+l+1})$.
\end{lemma}

\begin{proof}

Let $G=\gamma_{k}(K_{k+l+1}), H=\tilde \gamma_{k}(K_{k+l+1})$

$k+1,l$-WL can run $k+1$-WL on each branches. Then $G_k$ and $H_k$ branches are distinguished. In other words, given a graph $F\in \{G, H\}$ branch $e_i$ and the corresponding binary number node $d_i$. $\forall j\in [k+1], \vu\in V^k$
\begin{equation}
    c(\psi_j(\vu, d_i))\to \mset{\psi_j(\vv, d_i)|\vv\in e_i^k} \to c_k(F[e_i])
\end{equation}
Therefore, $k+1,l$-WL is more expressive than ignore branches (nodes with initial color $c_{b}$) while label $d_i$ with $c_k(F[e_i])$, which is equivalent to running $k+1,l$-WL on $\omega(K_{k+l+1}), \tilde \omega(K_{k+l+1})$. As $2, k+l-1$-WL can distinguish $\omega(K_{k+l+1}), \tilde \omega(K_{k+l+1})$, and $2, k+l-1$-WL $\preceq$ $k+1, l$-WL, $k+1, l$-WL can distinguish these two graphs.
\end{proof}

\section{Relationship with other works}\label{relationshipwithother}

In this section, we give a comprehensive discussion on the relationship between our framework and existing works. $k,l$-WL framework unifies (local) relational pooling based methods, a considerably wide range of subgraph GNNs and other works like OSAN \citep{OSAN}, establishing the connections between these fields.

First, we give a systematical discussion on our $k,l$-WL framework with GNN extensions, which is ommited in \cref{sectionunify}. \citep{TheoreticalComparison} points out that GNN extensions can be classified into higher order WL, counting substructures, injecting local information, and marking neighbors. While this classification is flexible (and not necessarily include all existing works), we can discuss the relationship between our work and these four kinds of GNN extensions and show that our method is more universal and superior.

\begin{itemize}
    \item Our model naturally incorporates traditional WL algorithms, including higher order $k$-WL. Our $k,l$-WL is a more comprehensive and universal hierarchy.
    \item Our model can fully count substructures (whose $\#$ nodes is within $k+l$). Also, statistically, our model does not need this upper bound number of labels to count substructures in practice (see proposition 4.10 as an example). Therefore, $k,l$-WL can naturally encode substructure information into its representations. Since we do not manually design which substructures to encode (as is done in GSN), it can capture more local information than manual methods. Additionally, it can also contain structure information beyond substructure counting.
    \item $k,l$-WL can naturally encode information up to a certain radius, which can be naturally achieved by performing message passing on labeled graph, or manually introducing the localized version. Note that in localized $k,l$-WL, the subgraph size even does not have to be fixed and can be easily transformed into other forms, including knowledge up to a flexible radius $r$. Users can design concrete methods under our framework according to their practical demands.
    \item Our method is based on and beyond marking. As is stated in \citep{TheoreticalComparison} and some other works, marking is the most efficient strategy to improve expressivity (at least in node-based subgraph GNNs). Our method incorporates GNN extensions with node markings (controlled by parameter $l$) and is the most universal, expressive one. Moreover, $k,l$-WL can simulate any other common GNN extensions with appropriate implementation.
\end{itemize}

In conclusion, our framework can incorporate all four extensions, thus incorporating an extensively wide range of GNN variations. 

Then we again emphasize the conclusions in \cref{sectionunify}.

\begin{itemize}
    \item $k,l$-WL can incorporate all relational pooling (RP) and local relational pooling methods, since node marking is the most general and expressive form and can simulate all other extensions \cite{TheoreticalComparison}. 
    \item $k,l$-WL incorporates a wide range of subgraph GNNs. \citet{AComplteHierarchy} shows that all node-based subgraph GNNs fall in one of 6 equivalent class of Subgraph Weisfeiler-Lehman Tests (SWL). Remarkably, SWL is exactly $1,1$-WL (and equivalently, $2,1$-WL) in our framework, which reveals the connection between our work and many other subgraph GNNs unified by \citet{AComplteHierarchy}. Moreover, $k,l$-WL also incorporate some other subgraph GNNs that are without the scope of SWL \cite{AComplteHierarchy}, such as $I^2$-GNN \cite{I2GNN}. Our $1,2$-WL is a slightly more powerful version than $I^2$-GNN since we consider all $2$-labeled tuples, while $I^2$-GNN only consider those connected $2$-labeled tuples. $1,2$-WL can distinguish some non-isomorphic graph pairs that SWL and 3-WL fail to discriminate, and the algorithm becomes even more powerful as we increase $k$ or $l$. 
    \item  A number of other works such as OSAN \citep{OSAN} are a strict class of our framework. There are still some works cannot be incorporated directly, though. For example, $k,l$-WL operates independently on different labeled graphs and does not include intersection between labeled (sub)graphs as in \citet{GNNAK}. Introducing inter-labeled-graph message passing will increase expressivity of $k,l$-WL, but at an intractable computation cost. It will also be complicated to analyze its theoretical expressivity if we introduce labeled graph interactions, which we leave for future work.
\end{itemize}

In conclusion, our $k,l$-WL and $k,l$-GNN can capture a large number of subgraph GNNs and relational pooling based method. Here are some examples. 

\begin{proposition}
    By definition, $1,1$-WL (or $1$-IDMPNN) incorporates ID-aware GNN, $1,2$-WL incorporates $I^2$-GNN, $1,k$-WL incorporates $k$-OSWL and $k$-OSAN.
\end{proposition}

\begin{proposition}
    $1,l$-WL (or $l$-IDMPNN) incorporate GSN that counts isomophism within $l$ nodes, when the hash function is strictly injective and powerful enough.
\end{proposition}

This is a natural conclusion since $l$-IDMPNN can count substructure within $l$ nodes. Note that GSN \citep{GSN} count non-isomorphic substructures in a hand-craft way, while our method can encode more information than counting substructures within $l$ nodes. 

\begin{proposition}
    $1,l$-WL capture K-hop GNN either by message passing with label of root node or by localization (perform localized $k,l$-WL within K-hop neighbors).
\end{proposition}

Since $1,l$-WL can distinguish any non-isomorphic subgraph if nodes of subgraph are all labeled, it captures all information obtained by K-hop GNN, including peripheral edge information.

\begin{proposition}
    $1,l$-WL capture node-based subgraph GNN, including node deletion and node marking. 
\end{proposition}

The proposition holds because node deletion and marking can both simulated by specific label mappings and update functions. In other words, our labeling method is the strongest form in sense of distinguishing non-isomorphic graphs when the number of labels is sufficient.

\section{An example to illustrate (localized) $k,l$-WL hierarchy}
\label{example}

\begin{figure}[ht]
\vskip 0.2in
\begin{center}
\centerline{\includegraphics[width=\columnwidth/2]{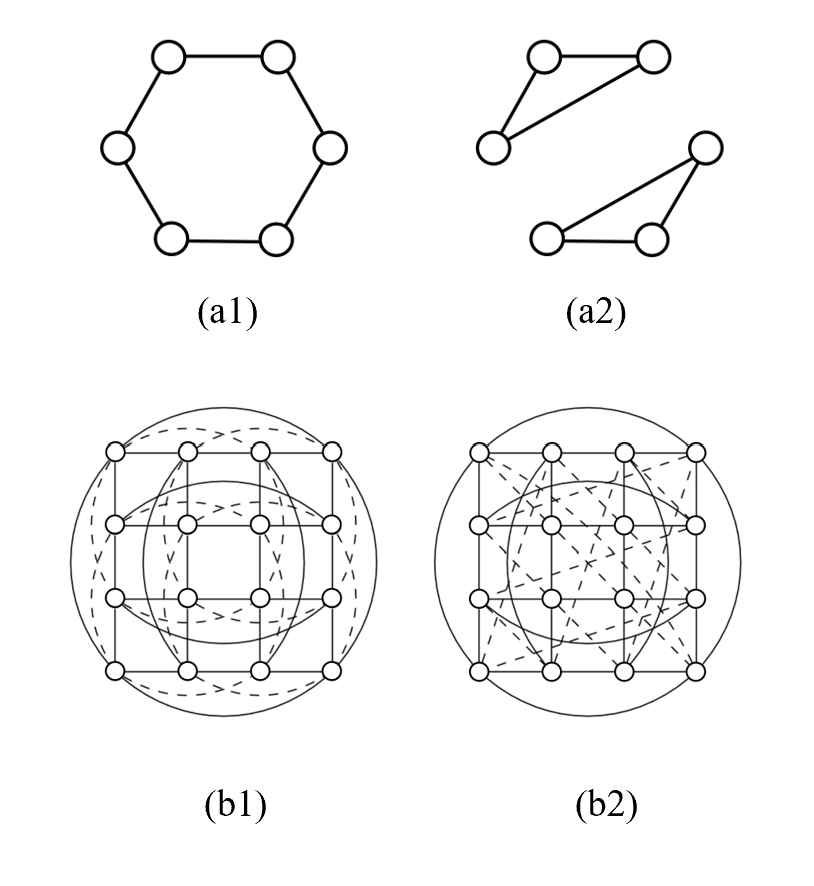}}
\caption{Two pairs of non-isomorphic graphs. 6-cycle (a1) and two 3-cycle (a2) are regular graphs which 1-WL and 2-WL fails to distinguish. 4x4 Rook’s graph (b1) and the Shrikhande graph (b2) are strongly regular graphs which 3-WL (2-FWL) fails to distinguish. Edges are dashed simply for visualization. The figure is modified from \citet{I2GNN}.}
\label{examplefigure}
\end{center}
\vskip -0.2in
\end{figure}

Here we give an example to show all parameters in $k,l$-WL hierarchy as well as localization have influence on expressiveness. In this section, we use $k,l$-FWL, which is equivalent to $k+1,l$-WL. The capacity of different $k,l$-WL and their localized variations for distinguishing these two pairs of non-isomorphic graph are summarized in \cref{regular-graph} and \cref{strongly-regular-graph}.

One can easily tell either increasing $k$ or increasing $l$ would lead to stronger expressivity. As for localization, we notice that if we extract 1-hop subgraph of any root node from (b1) and (b2), we'll get (a2) and (a1) respectively, except that all 6 nodes are connected with the root node. We can prove that if we perform $k,l$-FWL on extracted 1-hop subgraph, the result will be the same as \cref{regular-graph}. Hence decreasing the subgraph size improves the expressive power of $1,1$-FWL, $2,0$-FWL and $2,1$-FWL in this case.


\begin{table}[t]
\caption{Capability of different $k,l$-FWL to distinguish 6-cycle (a1) and two 3-cycle (a2)}
\label{regular-graph}
\vskip 0.15in
\begin{center}
\begin{small}
\begin{sc}
\begin{tabular}{cccc}
\toprule
        & l=0 & l=1 & l=2\\
\midrule
    k=1 &  \XSolidBrush & \Checkmark & \Checkmark\\
    k=2 & \Checkmark & \Checkmark & \Checkmark\\

\bottomrule
\end{tabular}
\end{sc}
\end{small}
\end{center}
\vskip -0.1in
\end{table}

\begin{table}[t]
\caption{Capability of different $k,l$-FWL to distinguish 4x4 Rook’s graph (b1) and the Shrikhande graph (b2)}
\label{strongly-regular-graph}
\vskip 0.15in
\begin{center}
\begin{small}
\begin{sc}
\begin{tabular}{cccc}
\toprule
        & l=0 & l=1 & l=2\\
\midrule
    k=1 &  \XSolidBrush & \XSolidBrush & \Checkmark\\
    k=2 & \XSolidBrush & \Checkmark & \Checkmark\\

\bottomrule
\end{tabular}
\end{sc}
\end{small}
\end{center}
\vskip -0.1in
\end{table}

\section{Discussion on explicit labels}\label{discussionID}

\paragraph{Insight of labels}

Here we empirically analyze the benefit of explicit introducing orders of nodes (i.e. labels or IDs). First, it expands the feature space considered in the histograms by introducing labels $\mathbf l$, which increase the number of isomorphism classes to better encode non-isomorphic states. Next, the algorithm introduces local asymmetry to the initialization states, allowing the algorithm to better distinguish local patterns, e.g. peripheral edges. Third, since the labels are unchanged in the iterations, they can be identified in all $k$-tuples, which is more powerful than simply counting histograms. In other words, each $k$-tuple can better acquire the source of information it obtains from neighbors, as well as the changing routes of node colors. 

Besides its effectiveness in graph isomorphism, our method turn out to be effective on real world tasks as well, which have less requirement for strictly powerful GI capability. We suggest that IDs enable the model to learn relative relations between nodes, and better capture structure information. For example, using only 2 labels can we distinguish 6-ring, while benzene rings play a important role in molecule properties.

\paragraph{Flexible number of labels} To align with $k$-WL hierarchy, we fix our number of labels across different subgraphs once the parameter $l$ of algorithm is given. However, a flexible number of labels within different subgraphs is a possible and practical way to improve performance, especially on real world tasks, see \cref{abalationflexible}. Corresponding theoretical analysis, however, will be harder, which we leave for future works.

\section{Theoretical analysis about two architectures of $k,l$-GNN}\label{architecturediscussion}

In this section, we theoretical analyze the expressivity of two architecture of $k,l$-GNN presented in \cref{architecture}. For the convenience of discussion, we currently ignore the effect of subgraph hierarchy. 

When the base encoder and structure encoder are both $k$-WL equivalent, the two structures have similar computation cost. When the base encoder is a higher order $k$-WL equivalent GNN, the base encoder itself has high complexity. Take PPGN as example, it has complexity of $O(n^2)$. If we use the architecture (a) to learn all $O(n^l)$ ordered subgraph features, the total complexity will be $O(n^{l+k})$. However, if we adopt structure (b) and use $k'$-WL ($k'<k$) equivalent GNN as structure encoder (e.g. $k'=1$ for MPNN), the total complexity will be $O(n^k+n^{l+k'})<O(n^{l+k})$. In this setting, the architecture (b) is equivalent to running $k,l$-WL and $k'$-WL in parallel, and further aggregate the information. This will make the network lose part of the expressivity compared with $k,l$-WL performed in structure \cref{architecture} (a). The above result suggests that expressivity of structure encoder will also influence the final GI capability.

\begin{proposition}
    The expressivity of architecture (b) is bounded by the union of $k',l$-WL and $k$-WL. Architecture (b) can improve expressivity of base encoders as long as there exist non-isomorphic graphs distinguishable by $k',l$-WL but not by $k$-WL.
\end{proposition}

\begin{proof}
    This holds since if both two algorithms fail to distinguish, the input of aggregator will be the same, resulting the model fails to distinguish.
\end{proof}

\begin{proposition}
    $k,l$-GNN implemented by architecture (a) is upper-bounded by $k,l$-WL. Suppose $k,l$-GNN implemented by architecture (b) uses a $k'$-WL ($k'\leq k$) equivalent structure encoder, the $k,l$-GNN (b) is upper-bounded by $k,l$-WL, and is less powerful than architecture (a).
\end{proposition}

\begin{proof}
    It's straight forward to verify that $k,l$-GNN with architecture (a) just perform the $k,l$-WL procedure. For architecture (a), the upper-bound can be achieved if and only if the forward function of base encoder and the aggregation function are injective. We then show that architecture (b) is upper-bounded by architecture (a) if the structure encoder in (b) is not more powerful than $k$-WL. By definition, there exists non-isomorphic pairs that is distinguishable by the base encoder but indistinguishable by structure encoder. The non-isomorphic graphs (b) can distinguish is the union of $k',l$-WL distinguishable graphs $G(k',l)$ and $k$-WL distinguishable graphs $G(k,0)$. Since both of $\{G(k',l)\}$ and $\{G(k,0)\}$ are a subset that $k,l$-WL can distinguish $\{G(k,l)\}$, i.e. $\{G(k',l)\}\supsetneqq \{G(k,l)\}, \{G(k,0)\}\supsetneqq \{G(k,l)\}$, according to the conclusions in \cref{Theory}. Therefore, architecture (b) cannot distinguish more non-isomorphic graphs than (a): $(\{G(k',l)\}\cup \{G(k,0)\})\supsetneqq \{G(k,;)\}$. 
\end{proof}

However, given that even $1,l$-WL is powerful enough in most cases, architecture (b) can still lift up expressivity of base encoder even if the structure encoder is MPNN. Meanwhile, it enjoys a much lower complexity compared with architecture (a). The above results theoretical enable us to improve base encoder's expressivity with only ID and lower complexity models like MPNN.

\section{Implementation details}\label{implementationdetail}

\paragraph{Explicitly embed IDs} As revealed in our paper, explicitly assigning IDs to nodes is a crucial step in improving GI power, which benefit both the initialization and update procedure of $k,l$-WL. However, previous work concerning relational pooling \citep{CountSubstructures, Murphy2019RelationalPF} do not break symmetry by explicitly introducing IDs.

In our implementation, a simple yet effective method to explicitly introduce IDs is using an embedding layer. For real world tasks, the embeddings are trainable to better adjust for the task. While in graph isomorphism tasks, the concrete values of network parameters do not matter, thus the embeddings of IDs can be fixed. In both trainable and fixed setting, the network is permutation invariance as long as we traverse all element in the symmetric group.

\paragraph{ID-MPNN and ID-GINE}

As illustrated in \cref{architecture} (b), we use two message-passing based GNN to learn graph features and ID features respectively. The hyper-parameters for these two encoders can be tuned independently.

Additionally, since there are a large amount of subgraphs, parallel computation efficiency should be considered. We implement MPNN that aggregates neighboring information by einsum of adjacency matrix and node feature tensors. Compared with official implementation of GINE, this implementation reveal approximately a $30\%$ of both time and memory consumption. 

In real world experiments, we perform message passing on the whole graph, i.e. set $m=n$, allowing our model to better capture global information of the graph. In substructure counting, we set $m=l$, which means only message passing on the labeled nodes. For graph isomorphism tasks, $m$ are adjusted according to the expressiveness hierarchy of localized $k,l$-WL.

\section{Experiment setting details}

On synthetic datasets, the training configuration is not limited due to strong expressivity of our model. We only report training details on real-world datasets.

\paragraph{QM9 dataset}
For QM9, we use $4$-IDMPNN. The hidden size is 64. We adopt parallel architecture, and 6-layer MPNN, 4-layer MPNN, 4-layer MPNN as base encoder, structure encoder and aggregator respectively. We use hierarchical and constrained based subgraph selection policy, with a selection rate of $0.05$. The initial learning rate
is 0.001 and we use ReduceLROnPlateau learning rate scheduler with a patience of 7 and a reduction
factor of 0.8. The maximum number of epochs is 200. We use AdamW optimizer and the batchsize is 32.

\paragraph{ZINC12k dataset}
For QM9, we use IDMPNN with different $l$. For the model with best performance, $l=4$ with a hidden size of 64. We adopt parallel architecture. Base encoder, structure encoder and aggregator are 5-layer MPNN, 4-layer MPNN, 4-layer MPNN, respectively. We use hierarchical and constrained based subgraph selection policy, with a selection rate of $0.1$. We use cosine learning rate with a warmup of 50 epochs, and the total number of epochs is 1000. The initial learning rate
is 0.001. We use AdamW optimizer and the batchsize is 32.

\paragraph{ogbg-molhiv dataset}
For ogbg-molhiv, we use $4$-IDMPNN with a hidden size 96. We adopt parallel architecture. Base encoder, structure encoder and aggregator are 6-layer MPNN, 4-layer MPNN, 4-layer MPNN, respectively. We use hierarchical and constrained based subgraph selection policy, with a selection rate of $0.05$. We use step scheduler with step size 20 and decay factor 0.5. The initial learning rate is 0.001 and the maximum number of epochs is 60. We use AdamW optimizer and the batchsize is 32.

\section{Ablation study for ID-MPNN}\label{ablation}

We conduct ablation study for ID-MPNN on the ZINC12k dataset. The following factors are considered: two architectures in \cref{architecture}, types of aggregator, number of IDs $l$, subgraph sampling and permutation sampling. Finally, we make some attempts on flexible number of IDs and use attention to label as aforementioned.

\paragraph{Comparison of two architectures}

Though in \cref{architecturediscussion} we theoretically proved that architecture (b) is strictly upper-bounded by architecture (a) in sense of graph isomorphism, experiments show that architecture (b) tend to have better performance in real world tasks. 

\begin{table}[t]
\caption{Comparison for two architecture of $l$-IDMPNN on ZINC.}
\label{ablationarchitecture}
\vskip 0.15in
\begin{center}
\begin{small}
\begin{sc}
\begin{tabular}{ccc}
\toprule
    Architecture & Test MAE\\
\midrule
    (a) &  $0.105 \pm 0.004$\\
    (b) & $0.083 \pm 0.003$\\

\bottomrule
\end{tabular}
\end{sc}
\end{small}
\end{center}
\vskip -0.1in
\end{table}

\cref{ablationarchitecture} presents results of $4$-IDMPNN's performance on ZINC12k with architecture (a) and (b) respectively. All other hyper-parameters and training configurations are set the same. This is because real world tasks like molecular property regression do not necessarily demand higher order GI capability. In architecture (b), the original input graph feature and the structure information provided by IDs are learned in a parallel way, enabling the structure encoder to better learn the relative structure information, independent of the original graph features.

\paragraph{Comparison of different aggregators}

In this section, we investigate the influence of aggregator in architecture (b). In graph isomorphism task, as long as the way we combine two features are injective, no downstream models are required to achieve full performance of $k,l$-WL. However, this may not be sufficient for real world tasks. To combine features output by base encoder and structure encoder, we use three methods: add, hadamard product (element-wise product) and concat. Then the combined feature is either directly passed to output linear layer, or further processed by a network (MPNN/MLP). Results are listed in \cref{ablationaggregator}.

\begin{table}[t]
\caption{Comparison for ID-MPNN with different aggregators on ZINC.}
\label{ablationaggregator}
\vskip 0.15in
\begin{center}
\begin{small}
\begin{sc}
\begin{tabular}{ccc}
\toprule
    Aggregator & Test MAE\\
\midrule
    add &  $0.125 \pm 0.011$\\
    add + MPNN & $0.083 \pm 0.003$\\
    hadamard product + MPNN & $0.084 \pm 0.002$\\
    concat + MPNN & $0.086 \pm 0.003$\\
\bottomrule
\end{tabular}
\end{sc}
\end{small}
\end{center}
\vskip -0.1in
\end{table}

The results indicate that although a downstream model in aggregator is not necessary in sense of GI, it's essential in graph regression task. A MPNN can better process the combination of features learned by base encoder and structure encoder. The way to combine these two features, however, is less important. All three methods (add, hadamard product and concat) work relatively well.

\paragraph{Comparison of $l$}

Theoretical analysis in \cref{Theory} has proved that graph isomorphism power of $k,l$-WL always increase w.r.t the $l$, the number of IDs. Now let's see how parameter $l$ in our $k,l$-GNN influence performance on real world tasks. 

\begin{table}[t]
\caption{Comparison for ID-MPNN with different number of IDs on ZINC.}
\label{ablationl}
\vskip 0.15in
\begin{center}
\begin{small}
\begin{sc}
\begin{tabular}{ccc}
\toprule
    Model & Test MAE\\
\midrule
    $3$-IDMPNN &  $0.085 \pm 0.003$\\
    $4$-IDMPNN & $0.083 \pm 0.003$\\
    $5$-IDMPNN & $0.089 \pm 0.004$\\
    $6$-IDMPNN & $0.094 \pm 0.004$\\
\bottomrule
\end{tabular}
\end{sc}
\end{small}
\end{center}
\vskip -0.1in
\end{table}

Results are in \cref{ablationl}, all other hyper-parameters and training configurations are set the same. We can conclude that though with higer power in GI, a larger number of IDs do not necessarily improve performance on real world tasks. This could be caused by several reasons, including a limited number of subgraphs selected by hierarchical and contraint-based policy for large $k$. The ablation study proves that our method can achieve satisfying performance with a relative small $l$.

\paragraph{Labeled set selection}

Now let's view the problem from relational pooling perspective. An equivalent way to select ordered subgraphs is to find a set of $l$ nodes and perform permutation on their assigned labels. Therefore, the total complexity of $k,l$-WL is $O(n^k \cdot n_s \cdot k!)$, where $n_s$ is the number of unordered subgraphs. Hence, labeled tuples selection in main text can be further separated into two stages: (1) selecting labeled node set, and (2) reducing number of permutation in the labeled set. Now we first look at the labeled set selection.

To evaluate influence of labeled set dropping, we conduct ablation study using $4$-IDMPNN with a smaller parameter amount. Results are shown in \cref{ablationsubgraph}. All hyper-pameters except the labeled set sample rate are set the same. We only drop labeled subgraphs when training, while retain all labeled subgraphs during inference. We use random selection policy in this section. Results are shown in \cref{ablationsubgraph}.

\begin{table}[t]
\caption{Comparison for ID-MPNN with different subgraph sample rate on ZINC.}
\label{ablationsubgraph}
\vskip 0.15in
\begin{center}
\begin{small}
\begin{sc}
\begin{tabular}{ccc}
\toprule
    Subgraph sample rate & Test MAE & Running time (s/epoch)\\
\midrule
    1.0 &  $0.162 \pm 0.002$ & 98\\
    0.5 & $0.161 \pm 0.003$ & 53\\
    0.25 & $0.153 \pm 0.004$ & 29\\
    0.1 & $0.150 \pm 0.004$ & 15\\
    0.05 & $0.154 \pm 0.004$ & 9\\
\bottomrule
\end{tabular}
\end{sc}
\end{small}
\end{center}
\vskip -0.1in
\end{table}

One can observe that even with a very small subgraph sample rate, our model can still achieve comparable or even better performance than using all subgraphs. This is partly due to the redundancy between subgraphs. Meanwhile, the time consumption is basically linear w.r.t subgraph sample rate, enabling us to lift up the training process of our model without hurting performance.

\paragraph{Permutation selection}

Now we further consider permutation drop for unordered subgraphs. A random selection of permutation $g\in S_l$ is a possible way due to its unbiasedness and consistency. However, this will break permutation invariance in practical, resulting in a large variance and poorer performance compared with adopting the full symmetric group, see \cref{ablationpermutation} below.

Another possible direction is to find a possible permutation group other than the symmetric group. \citep{PG-GNN} proposed a smaller permutation method that is able to capture 2-ary relationship. However, its experimental performance are limited, and there are still a large space to theoretically explore the permutation design.

In our experiments, $4$-IDMPNN with a larger parameter amount is used as test model, and the subgraph sample rate is $0.1$. Results are shown in \cref{ablationpermutation}. Different from subgraph drop, a low permutation sample rate will catastrophically hurt the performance. Note that models with fewer permutations tend to be more likely to overfit. Statistically, a permutation drop is an unbiased estimator, but is permutation sensitive in every single run, resulting a larger variance on the test set. The above result again emphasizes the importance of permutation invariance for GNNs.

\begin{table}[t]
\caption{Comparison for ID-MPNN with different permutation sample rate on ZINC.}
\label{ablationpermutation}
\vskip 0.15in
\begin{center}
\begin{small}
\begin{sc}
\begin{tabular}{ccc}
\toprule
    Permutation sample rate & Train MAE & Test MAE\\
\midrule
    1.0 &  $0.008 \pm 0.001$ & $0.083\pm 0.003$ \\
    0.9 & $0.010 \pm 0.003$ & $0.085\pm 0.004$\\
    0.5 & $0.009 \pm 0.002$ & $0.097\pm 0.008$\\
    0.2 & $0.007 \pm 0.001$ & $0.141\pm 0.018$\\
\bottomrule
\end{tabular}
\end{sc}
\end{small}
\end{center}
\vskip -0.1in
\end{table}

\begin{table}[t]
\caption{Comparison for test MAE and running time of different methods on ZINC}
\label{abalationrunningtime}
\vskip 0.15in
\begin{center}
\begin{small}
\begin{sc}
\begin{tabular}{ccc}
\toprule
    Model & Test MAE & Running time (s/epoch)\\
\midrule
    GraphSAGE &  $0.398 \pm 0.002$ & 16.61\\
    DeepLRP-7-1 & $0.223 \pm 0.008$ & 72\\
     PPGN & $0.256\pm 0.054$ & 334.69\\
     PG-GNN & $0.282 \pm 0.011$ & 6.92 \\

\midrule
    4-IDMPNN & $0.083\pm 0.003$ & 17 \\
    3-IDMPNN & $0.085\pm 0.003$ & 12\\
\bottomrule
\end{tabular}
\end{sc}
\end{small}
\end{center}
\vskip -0.1in
\end{table}

\paragraph{Comparison with other models applying local relational pooling}

For comparison, we select two permutation invariant GNNs using local relational pooling (and can be incorporated into our $k,l$-GNN framework): PG-GNN \citep{PG-GNN} and DeepLRP \citep{CountSubstructures}, evaluating their test MAE and computation efficiency. GraphSAGE \citep{GraphSage} and PPGN are also included as baselines. It turns out that our instances have significantly better performance as well as computational efficient. The result is shown in \cref{abalationrunningtime}.

\begin{table}[h]
\caption{Comparison for data-driven methods on ZINC}
\label{abalationflexible}
\vskip 0.15in
\begin{center}
\begin{small}
\begin{sc}
\begin{tabular}{ccc}
\toprule
    Method & Test MAE \\
\midrule
    GIN &  $0.163 \pm 0.004$\\
    OSAN & $0.187 \pm 0.004$ \\ 
    Flexible label & $0.128\pm 0.011$\\
    Attention-based label & $0.135 \pm 0.008$\\
\bottomrule
\end{tabular}
\end{sc}
\end{small}
\end{center}
\vskip -0.1in
\end{table}

\paragraph{Other experiments}

Based on Implicit-MLE, we propose another method to let our model learn to label. We can measure "importance" of nodes by calculating self-attention of node features obtained by a base encoder (e.g. a MPNN), and label according to the importance. This method allow us to label arbitrary number of labels in $O(1)$ complexity, since no permutation is demand to keep invariance. However, this method is not applicable in GI tasks when graphs reveal strong symmetry, since the representation are bounded by expressivity of base encoder. See below for more experimental results.

Here we report 2 more methods to label subgraphs in \cref{abalationflexible}: using a flexible number of labels across subgraphs, and the aforementioned labeling nodes according to the self-attention weight of nodes. Concretely, the former one label each substructure (pre-calculated node cluster) according to the size of cluster. The latter learns the ID in a data-driven manner, hence do not need any permutations. We choose OSAN with I-MLE as baseline model. Our results show that these methods slight improve performance of base encoder, but still have a wide design space for future work.

\end{document}